\documentclass[journal]{IEEEtran}

\usepackage{amsmath,amsfonts,amssymb,amsthm}
\usepackage{graphicx}
\usepackage{algorithmic}
\usepackage{bbm}
\usepackage{mathrsfs}
\usepackage{dsfont}
\usepackage{multirow}
\usepackage{mathtools}
\usepackage{booktabs}
\usepackage{float}
\usepackage{multicol}
\usepackage{svg}
\usepackage{cuted}
\usepackage{pifont}
\usepackage[table]{xcolor}
\usepackage{float}

\usepackage{tabularx}
\usepackage{array}

\newcommand{\cmark}{\ding{51}} 
\newcommand{\xmark}{\ding{55}} 
\definecolor{plotblue}{HTML}{4F9DDC}  
\definecolor{plotyellow}{HTML}{F5DEB3} 
\definecolor{plotgreen}{HTML}{1E8449}  
\usepackage{makecell}

\usepackage{arydshln}
\usepackage{url}
\usepackage{xr-hyper}
\usepackage{svg}
\usepackage{subfiles}
\usepackage{hyperref}

\newtheorem{theorem}{Theorem}[section]

\newtheorem{prop}[theorem]{Proposition}

\usepackage{setspace}
\begin{document}
\title{Design Choices in Splitting-Based Self-Supervised Sparse-View CT Reconstruction}
\author{%
  Nadja~Gruber\IEEEauthorrefmark{1}\thanks{Corresponding authors: nadja.gruber@uibk.ac.at, johannes.schwab@fh-kufstein.ac.at}%
  , Lukas~Neumann\IEEEauthorrefmark{2}%
  , Ander~Biguri\IEEEauthorrefmark{3}%
  , Gyeongha~Hwang\IEEEauthorrefmark{4}%
  , Markus~Haltmeier\IEEEauthorrefmark{5}%
  , Johannes~Schwab\IEEEauthorrefmark{6}%
  \thanks{\IEEEauthorrefmark{1} Department of Computer Science, University of Innsbruck, Austria.}%
  \thanks{\IEEEauthorrefmark{2} Institute of Basic Sciences in Engineering Science, University of Innsbruck, Austria.}%
  \thanks{\IEEEauthorrefmark{3} Department of Applied Mathematics and Theoretical Physics, University of Cambridge, UK.}%
  \thanks{\IEEEauthorrefmark{4} Department of Mathematics, Yeungnam University, Gyeongsan, Korea.}%
  \thanks{\IEEEauthorrefmark{5} Department of Mathematics, University of Innsbruck, Austria.}%
  \thanks{\IEEEauthorrefmark{6} University of Applied Sciences Kufstein, Austria.}%
}

\maketitle
\begin{abstract}
Self-supervised data splitting has emerged as a promising paradigm for sparse-view CT reconstruction, enabling training from incomplete measurements without fully sampled ground truth. However, the influence of key design choices, including partitioning strategy, preprocessing, and inference, remains insufficiently understood. In this work, we introduce a unified framework that decomposes splitting-based reconstruction into these three components, enabling controlled comparison of existing methods and two incremental extensions: multi-partition splitting and an alternative inference strategy. Experiments on simulated LoDoPaB-CT data under independent and correlated noise, together with validation on the real-world 2DeteCT dataset, show that the optimal partitioning strategy strongly depends on the measurement noise structure. Lattice-based splitting performs favorably under independent noise, whereas angular masking is more robust under correlated noise and real measured data. Multi-partition splitting consistently improves over pure projection-wise splitting in several settings. Complementary perceptual and structural metrics, including LPIPS and HaarPSI, reveal differences between masking strategies that are less apparent from PSNR and SSIM alone. These results provide practical guidelines for designing self-supervised sparse-view CT reconstruction methods and highlight the limitations of common independence assumptions in realistic imaging environments.
\end{abstract}

\section{Introduction}\label{sec:intro}
In imaging modalities such as X-ray computed tomography (CT), cryo-electron microscopy, and photoacoustic imaging, images are reconstructed from indirect and often incomplete measurements. In practice, acquisition constraints such as limited scan time, radiation dose, or physical limitations lead to severely undersampled inverse problems of the form
\begin{align}\label{inverse}
y = \mathcal{N}(A x),
\end{align}
where $A: X \rightarrow Y$ denotes the forward operator, $\mathcal{N}$ represents a noise corruption process, $x \in X$ is the unknown image, and $y \in Y$ are the measured data. Such inverse problems are often ill-posed due to the ill-conditioned nature of $A$ and the presence of noise, meaning additional prior information is required for stable reconstruction.

Classical variational methods incorporate handcrafted regularizers such as smoothness or total variation \cite{scherzer2009variational,rudin1992nonlinear,bredies2010total,zhang2014accurate}. While theoretically well established, their performance degrades in strongly undersampled regimes and depends heavily on careful parameter selection.
More recently, deep learning approaches have achieved state-of-the-art performance by learning image priors directly from data. However, most methods rely on supervised training with paired measurements and ground-truth images, which are often unavailable in practical clinical or industrial imaging settings \cite{belthangady2019applications}.

To overcome this limitation, self-supervised learning approaches have emerged to exploit structure within the measurements themselves. A prominent class of such methods is based on measurement splitting, where observations are divided into complementary subsets that supervise each other. Originating from denoising formulations such as Noise2Noise \cite{lehtinen2018noise2noise}, Noise2Void \cite{krull2019noise2void}, and Noise2Self \cite{batson2019noise2self}, these ideas have been extended to general inverse problems, including sparse-view CT reconstruction \cite{hendriksen2020noise2inverse, gruber2024sparse2inverse, unal2021self}. Related approaches also exploit equivariance priors to enable learning from incomplete data \cite{chen2021equivariant, chen2022robust, schut2026equivariance2inverse}.

Despite this progress, key design choices in splitting-based reconstruction methods have rarely been evaluated in a unified and controlled framework. In particular, key components such as data partitioning, preprocessing of masked measurements, and inference strategies are typically studied within specific frameworks, making it difficult to isolate their individual effects. As a result, it remains unclear which design choices are most critical and how they interact under different noise assumptions.
This issue is especially critical in X-ray CT, where detector physics introduce structured noise correlations. In particular, scintillator blur induces spatially correlated noise between neighboring detector elements, violating the independence assumptions underlying many splitting-based methods \cite{graas2025scintillator}. This raises an important question: how robust are splitting strategies when the assumption of measurement independence is violated?

In this work, we address this question by introducing a unified framework that decomposes splitting-based self-supervised reconstruction into three core components: partitioning, preprocessing, and inference. This framework enables a controlled comparison of existing methods and new extensions across different noise regimes, including correlated scintillator-induced noise. 

Furthermore, evaluating self-supervised methods under realistic noise is highly sensitive to the choice of metric. While PSNR is widely used, it often fails to reflect perceptual or diagnostically relevant structural properties. Motivated by findings in \cite{breger2024study}, we complement standard error-based metrics with advanced perceptual measures, such as LPIPS~\cite{zhang2018unreasonable} and HaarPSI~\cite{reisenhofer2018haar}, aiming to capture subtle reconstruction artifacts that may not be reflected by conventional metrics.

The main contributions of this work are:
\begin{itemize}
\item A unified framework that decomposes splitting-based reconstruction into partitioning, preprocessing, and inference, providing a common perspective on existing self-supervised CT reconstruction methods.
\item A systematic benchmarking of these components under independent and correlated noise regimes, including scintillator-induced correlations, using both synthetic LoDoPaB-CT data and the real-world 2DeteCT dataset \cite{kiss20232detect}, evaluated with error-based and perceptual metrics.
\item The introduction of multi-partition masking and a novel inference strategy for splitting-based reconstruction, allowing complementary measurement partitions to be incorporated during training and inference.
\item An analysis showing that the optimal splitting strategy strongly depends on the noise structure, with angular masking being more robust under correlated noise, while lattice masking is preferred under independent noise.
\item Practical design guidelines for self-supervised CT reconstruction, summarizing how partitioning, preprocessing, and inference strategies should be selected under different noise regimes. 
\end{itemize}

\begin{figure*}
    \centering
\includegraphics[width=0.7\linewidth]{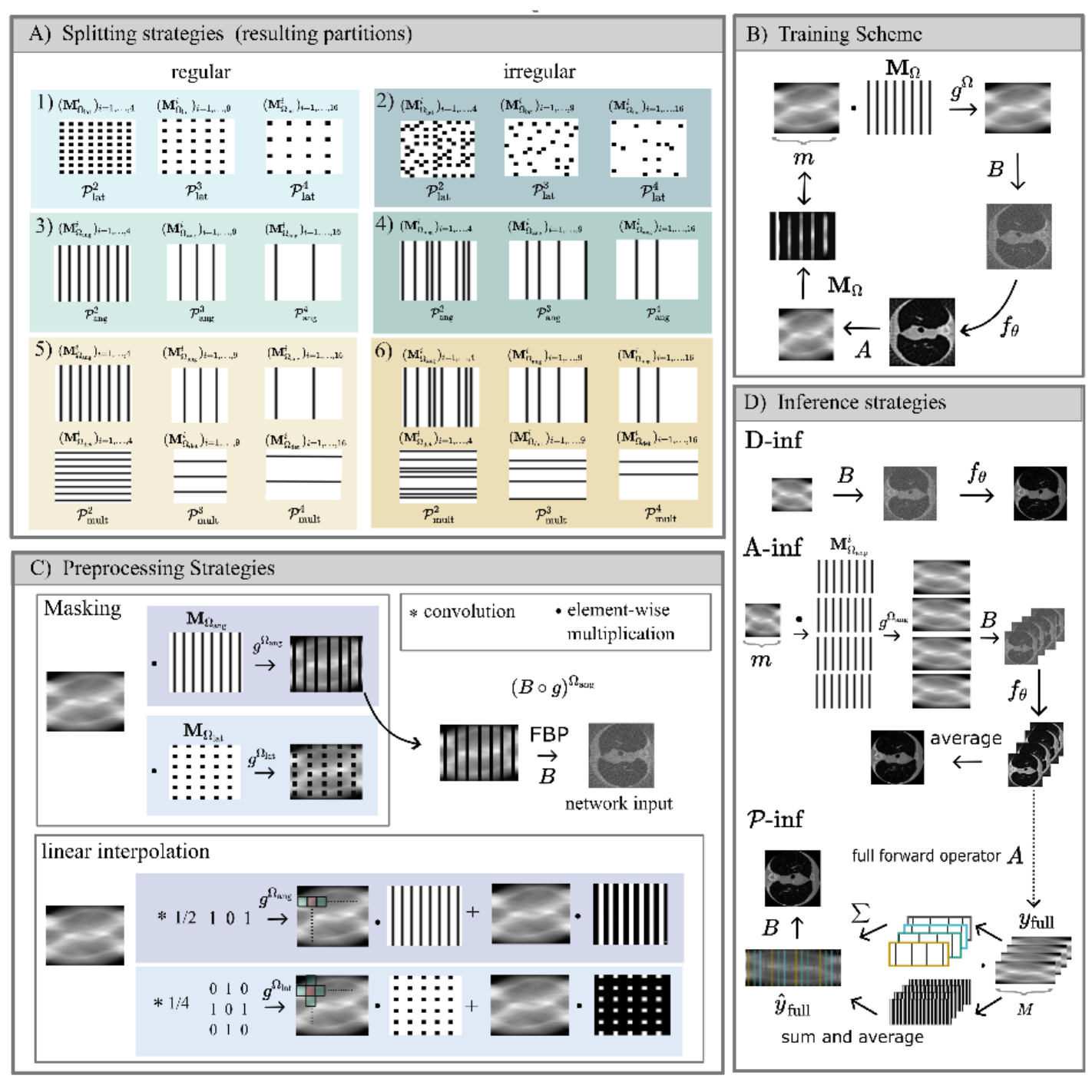}
\caption{Unified framework components for self-supervised CT reconstruction. 
\textbf{(A) Splitting strategies:} Regular and irregular partitions for lattice ($\mathcal{P}_{\text{lat}}$), angular ($\mathcal{P}_{\text{ang}}$), and multi-partition ($\mathcal{P}_{\text{mult}}$) masking. 
\textbf{(B) Training scheme:} Data-domain loss enforcing consistency via forward projection $A$ on held-out coordinates after FBP reconstruction $B$. 
\textbf{(C) Preprocessing:} Comparison of zero-filling versus linear interpolation. 
\textbf{(D) Inference strategies:} Direct (\textit{D-inf}), Average (\textit{A-inf}) via image ensembling, and $\mathcal{P}$-invariant (\textit{P-inf}) via dense sinogram refinement. 
Sinogram aspect ratios are modified for clarity.}
    \label{fig:overview} \label{overview}
\end{figure*}

\section{Methods}\label{sec:method}
In this section, we describe the proposed self-supervised splitting-based learning and inference strategies for sparse-view CT reconstruction considered in our study. 
We restrict the formulation to a 2D parallel-beam geometry. However, the overall concept can be extended straightforwardly to more general geometries in both 2D and 3D and inverse problems.
\subsection{Notation}
Let $x \in X$ represent a discretized ground-truth image, and let $y = \mathcal{N}(Ax) \in Y = \mathbb{R}^{m \times n}$ denote the noisy, sparse-view CT measurements (sinogram). Here, $m$ corresponds to the number of equidistant projection angles in $[0, \pi - \pi/m]$, and $n$ denotes the number of detector elements. The forward operator $A: X \rightarrow Y$ maps an image to its clean tomographic projections, while $\mathcal{N}: Y \rightarrow Y$ models the measurement noise process, which may be non-additive. The objective is to recover the unknown clean image $x$ from $y$, where the number of projection angles $m$ is significantly smaller than the image dimensions.
\subsection{$\mathcal{P}$-Invariance}\label{invariance}

Throughout this section, we assume that $\Omega \subset \{1,\dots, m\}\times\{1,\dots,n\}\eqqcolon [m]\times[n]$ is any subset of the indices in the sinogram space $\mathbb{R}^{m\times n}$. For such $\Omega$, we define the sinogram mask $\mathbf{M}_\Omega \in \{0,1\}^{m\times n}$ by
\begin{align}
(\mathbf{M}_\Omega)_{i,j} \coloneqq 
\begin{cases}
    1, \quad &(i,j)\in\Omega\\
    0, &\text{else,}
\end{cases}
\end{align}
and the complement of a set $\Omega$ by $\Omega^c = ([m]\times[n])\setminus\Omega$ such that $\mathbf{M}_{\Omega^c} = \mathbf{1}-\mathbf{M}_\Omega.$ All multiplications with masking matrices $\mathbf{M}_{\Omega}$ are to be understood point-wise (Hadamard product).

The core idea behind self-supervised splitting methods is the use of $\mathcal{P}$-invariant mappings, as introduced in the Noise2Self framework \cite{batson2019noise2self}. Specifically, these mappings are designed such that certain groups of output coordinates are independent of the corresponding groups of input coordinates.

Formally, let $\mathcal{P}$ be a partition of $[m]\times[n]$, which is a collection of disjoint subsets $\Omega\subset[m]\times[n]$ satisfying $\biguplus_{\Omega\in\mathcal{P}}\Omega = [m]\times[n]$. For $\Omega \in \mathcal{P}$ and data $z\in\mathbb{R}^{m\times n}$, a function $f^\mathcal{P}:\mathbb{R}^{m\times n} \rightarrow \mathbb{R}^{m\times n}$ is called $\Omega$-invariant if $\mathbf{M}_\Omega f^\mathcal{P}(z)$ depends only on $\mathbf{M}_{\Omega^c}z$. It is called $\mathcal{P}$-invariant if it is $\Omega$-invariant for every $\Omega\in \mathcal{P}$.

A general way to construct a $\mathcal{P}$-invariant function is the following: for an arbitrary function $f:\mathbb{R}^{m\times n}\rightarrow\mathbb{R}^{m\times n}$ and any partition $\mathcal{P}$, a $\mathcal{P}$-invariant map can be defined by
\begin{align}\label{eq:assembly}
\forall y \in\mathbb{R}^{m\times n}:\,\, f^\mathcal{P}(y)= \sum_{\Omega\in\mathcal{P}} \mathbf{M}_\Omega f( \mathbf{M}_{\Omega^c} (y)).
\end{align}
Then, by construction, $f^\mathcal{P}$ is a $\mathcal{P}$-invariant function. For denoising problems, where $A=\mathrm{Id}$, \eqref{eq:assembly} can be directly used for prediction, as masking operates within a single domain. In contrast, for a general operator $A$, data splitting and reconstruction occur in different domains, so masking in the measurement space does not translate directly to masking in the image space, making inference less straightforward.

For sparse-view CT, it remains unclear which partitioning strategies are best suited to different noise regimes, motivating the systematic comparison presented in this work. In practice, the abstract partition $\mathcal{P}$ defines the specific geometric splitting strategy applied to the sinogram. In this work, we analyze different choices for $\mathcal{P}$—such as angular, projection-wise, and lattice-based splitting (detailed in Section~\ref{splitting})—to evaluate how these partitioning choices interact with different physical noise characteristics, which can vary significantly depending on the underlying inverse problem.

\subsection{Self-supervised Splitting-based Methods}

Applied to inverse problems of the form \eqref{inverse}, splitting-based approaches partition the projection domain into subsets $\mathcal{P}$, under the assumption that the noise is statistically independent across the corresponding measurement groups. For each subset $\Omega \in \mathcal{P}$, the measurements on $\Omega$ are withheld from the network input and used only for supervision, while the complementary measurements on $\Omega^c$ are used to form an input reconstruction. A neural network $f_\theta$ is then trained in the reconstruction domain and evaluated through a data-domain consistency loss on the held-out measurements (Figure~\ref{fig:overview}B). Hendriksen et al.~\cite{hendriksen-2021-tomos} demonstrated that this Noise2Self variant achieves superior performance compared to alternative self-supervised strategies. However, in practical CT systems, detector effects such as scintillator-induced correlations can violate the independent noise assumption.

To construct the network input from the available measurements, we consider, for each $\Omega \in \mathcal{P}$, a linear reconstruction pipeline
$B \circ g^{\Omega^c} : Y \rightarrow X,$
where $B$ denotes filtered backprojection (FBP). The preprocessing function $g^{\Omega^c}$ maps the masked sinogram $\mathbf{M}_{\Omega^c}y$ to a complete sinogram representation before reconstruction. In this work, $g^{\Omega^c}$ either leaves the missing entries as zero or estimates them from the available measurements in $\Omega^c$.

\subsubsection{Preprocessing Strategies}\label{sec:preprocessing}
In this manuscript, we consider two choices for the preprocessing operator $g^{\Omega^c}$, which maps the masked sinogram $\mathbf{M}*{\Omega^c}y$ to a complete sinogram representation before applying FBP:
\begin{itemize}
\item \textbf{Zero-filling:} missing entries in $\Omega$ are left as zero, i.e., $g^{\Omega^c}(\mathbf{M}_{\Omega^c}y)=\mathbf{M}_{\Omega^c}y$.
\item \textbf{Linear interpolation:} missing entries in $\Omega$ are estimated from the available measurements in $\Omega^c$ using linear interpolation.
\end{itemize}
The described options are visualized in Figure~\ref{fig:overview}\textbf{C}.

\subsubsection{Network Training}
We now turn to the training strategy employed in splitting-based schemes. An overview is provided in Figure~\ref{fig:overview} \textbf{B}. The objective is to train a neural network $f_\theta : X \rightarrow X$ by minimizing a loss function defined in the data domain over the training set. This can be formulated as
\[
\mathcal{L}^Y(\theta) = \frac{1}{\lvert\mathcal{P}\rvert}\sum_{\Omega\in\mathcal{P}}\left\| \textbf{M}_\Omega A\circ f_\theta\left((B \circ g^{\Omega^c})(\mathbf{M}_{\Omega^c} y)\right) - \mathbf{M}_\Omega y \right\|^2 .
\]
We denote the minimizer of $\mathcal{L}^Y$ by $\theta^\ast$.
In our experiments, we employ a U-Net-type \cite{ronneberger2015u} architecture for $f_\theta$. It has been shown \cite{gruber2024sparse2inverse} that minimizing the loss in the data domain is particularly advantageous in sparse-view settings, as opposed to minimizing it in the reconstruction domain, i.e.
\[
\mathcal{L}^X(\theta) = \frac{1}{\lvert\mathcal{P}\rvert}\sum_{\Omega\in\mathcal{P}}\left\| f_\theta \left(B \circ g^{\Omega^c}(\mathbf{M}_{\Omega^c} y)\right) - B \circ g^{\Omega}(\mathbf{M}_\Omega y)\right\|^2 ,
\]
which was proposed in \cite{hendriksen2020noise2inverse} for full-view CT reconstruction. While $\mathcal{L}^X$ enables direct supervision in the image domain, its targets are themselves affected by undersampling artifacts. In sparse-view CT, these artifacts become severe and may be learned by the network. In contrast, $\mathcal{L}^Y$ enforces consistency with the measured projections and has been shown to yield superior reconstruction quality in sparse-view settings \cite{gruber2024sparse2inverse}.

\subsection{Inference Strategies}\label{sec:inference}
In the denoising and inverse problems literature on splitting-based self-supervised learning schemes \cite{batson2019noise2self,hendriksen2020noise2inverse,unal2021self}, two inference strategies are commonly used to obtain the final reconstruction $\hat{x}$: direct inference (D-inf) and average inference (A-inf). In addition, we consider a proposed extension, referred to as $\mathcal{P}$-invariant inference ($\mathcal{P}$-inf).
 We refer to them as:
\begin{itemize}
    \item \textbf{Direct inference (D-inf)},
    \item \textbf{Average inference (A-inf)},
    \item \textbf{$\mathcal{P}$-invariant inference ($\mathcal{P}$-inf)}.
\end{itemize}
In the following, the three strategies will be explained in detail.
\subsubsection{Direct Inference - D-inf}
We begin with D-inf, which has been used as a standard choice in both denoising \cite{batson2019noise2self} and image reconstruction \cite{unal2021self}. The trained neural network $f_{\theta^\ast}$ is applied directly to the measured data:
\begin{align}\label{eq:single_inf}
\hat{x} = f_{\theta^\ast}(B y).
\end{align}
\subsubsection{Average Inference - A-inf}
The second strategy, A-inf, introduced in Noise2Inverse \cite{hendriksen2020noise2inverse}, applies $f_{\theta^\ast}$ to multiple sub-reconstructions and averages the results:
\begin{align}\label{eq:multi_inf}
    \hat{x} = \frac{1}{|\mathcal{P}|}\sum_{\Omega \in \mathcal{P}}
    f_{\theta^\ast}\bigl(B \circ g^{\Omega^c} (\mathbf{M}_{\Omega}y)\bigr).
\end{align}
Unlike D-inf, this corresponds to an ensemble over complementary data splits.
Note that for both D-inf and A-inf, the forward projection $A(\hat{x})$ depends on the measured data $y$ at indices $\Omega$, i.e., $\textbf{M}_\Omega A(\hat{x})$ depends on $\textbf{M}_\Omega y$. Hence, $A(\hat{x})$ is not $\mathcal{P}$-invariant. This observation motivates the design of an inference strategy that enforces $\mathcal{P}$-invariance.

\subsubsection{$\mathcal{P}$-invariant inference}
We modify A-inf to aggregate sub-reconstructions in sinogram space rather than image space, ensuring the $\Omega$-component is never averaged into itself and thus restoring $\mathcal{P}$-invariance. This is enabled by an augmented sinogram representation via finer angular sampling, which reduces null-space ambiguities of $A$ and better preserves the image-space information learned by $f_\theta$.
Let $m$ denote the number of projection angles in the original sinogram and choose a refinement factor $r \in \mathbb{N}$ such that $M = r m$. This induces a refined sinogram grid of size $M \times n$.
We define a refinement map $\iota(i,j) = (r(i-1)+1, j)$,
which embeds the original sampling locations into the refined grid. Further, let $\bar{\Omega}\subset[M]\times[n]$. The observed measurements correspond to a binary mask $\mathbf{M}_{\mathrm{\bar{\Omega}}}$ on the refined grid, defined by
$\mathbf{M}_{\mathrm{\bar{\Omega}}}(\iota(i,j)) = 1, \quad \mathbf{M}_{\mathrm{\bar{\Omega}}} = 0 \text{ otherwise},
$
for $(i,j)\in\Omega.$
Let $\mathcal{A}: X \rightarrow \mathbb{R}^{M \times n}$ be the forward operator with dense angular sampling, whose nullspace is smaller than that of the corresponding sparse-view operator.
We first construct a coarse prediction on the observed angles:
\[
\hat{y}_{\bar{\Omega}} =
\sum_{\Omega \in \mathcal{P}}
\mathbf{M}_{\bar{\Omega}}\,
\mathcal{A}\circ f_{\theta^\ast}
\big((B \circ g^{\Omega^c})(\mathbf{M}_{\Omega} y)\big).
\]
For the unobserved part of the data, we define
\[
\hat{y}_{\bar{\Omega}^c} =
\frac{1}{|\mathcal{P}|}
\sum_{\Omega \in \mathcal{P}}
\mathbf{M}_{\bar{\Omega}^c}\,
\mathcal{A} \circ f_{\theta^\ast}
\big((B \circ g^{\Omega^c})(\mathbf{M}_{\Omega} y)\big).
\]
Finally, the reconstructed image is obtained via
\[
\hat{x} = B(\hat{y}_{\bar{\Omega}} + \hat{y}_{\bar{\Omega}^c}).
\]
A graphical description of the three inference strategies can be seen in Figure~\ref{fig:overview} \textbf{D}.
This inference is $\mathcal{P}$-invariant, as entries in $\bar{\Omega}^c$ depend only on the observed measurements in $\bar{\Omega}$.

\subsection{Compared Splitting Strategies}\label{splitting}

In self-supervised reconstruction methods such as Proj2Proj \cite{unal2021self} and Noise2Inverse \cite{hendriksen2020noise2inverse}, a subset $\Omega \subset [m]\times[n]$ of the measured data is masked during training. For example, \cite{unal2021self} applies periodic masking of detector elements, while angular splitting \cite{hendriksen2020noise2inverse, gruber2024sparse2inverse} removes entire projection angles. In the following, $d \in \{2, 3, 4\}$ denotes the base partitioning factor and $s=d^2\in\{4,9,16\}$ denotes the actual total number of splits within a geometric partition, determining the effective sampling density per split. This convention is used for all evaluated partitioning geometries to enable comparisons at matched number of splits.

\subsubsection{Single-Partition Splitting}
We refer to single-partition masking as the setting in which one fixed partition is used. This is the standard choice in existing self-supervised reconstruction methods, including Proj2Proj \cite{unal2021self} and Noise2Inverse \cite{hendriksen2020noise2inverse}, as well as our previous work \cite{gruber2024sparse2inverse}. We later extend this setting to multiple partitions.

\noindent\textit{Regular Lattice Splitting ($\mathcal{P}_{\textrm{lat}}^d$).}
This strategy, introduced in \cite{unal2021self}, partitions the sinogram into a regular grid of non-overlapping patches of size $d \times d$, yielding a total of $s=d^2$ splits. Within each patch, all pixels sharing the same relative position are grouped into a subset. Formally, for a given index $d$, the partition is denoted as $\mathcal{P}_{\textrm{lat}}^d$ and is given by
\[
\mathcal{P}_{\textrm{lat}}^d
=
\left\{
\Omega_{\textrm{lat}}^{(i,h)} \;:\; i,h = 1,\dots,d
\right\},
\]
with
\[
\Omega_{\textrm{lat}}^{(i,h)}
=
\left\{
(j,l) \in [m]\times[n]
\;\middle|\;
j \bmod d = i,\;
l \bmod d = h
\right\}.
\]
Thus, each subset contains all pixels that share the same relative location within each patch. Examples for different values of $d$ are shown in Fig.~\ref{fig:overview} \textbf{A} 1). In the experiments and plots, we evaluate $\mathcal{P}_{\textrm{lat}}^d$ with $d \in \{2, 3, 4\}$, respectively.

\noindent\textit{Irregular Lattice Splitting.}
In contrast, irregular lattice masking randomly selects one pixel per $d \times d$ patch, rather than fixing a deterministic position. This yields a non-deterministic partition $\mathcal{P}_{\textrm{lat}}^d$, as illustrated in Fig.~\ref{fig:overview} \textbf{A} 2).

\noindent\textit{Regular Angular Splitting ($\mathcal{P}_{\textrm{ang}}^d$):}
Following previous work \cite{hendriksen2020noise2inverse, gruber2024sparse2inverse}, masking can also be performed along entire projection angles. To maintain notation consistency across all benchmarked methods, we define the total number of angular splits via $s = d^2$. For a given index $d \in \{2, 3, 4\}$, subsets of projection directions are defined as
\[
\Omega_{\textrm{ang}}^i = \{j \in [m] \mid j \bmod d^2 = i\} \times [n], 
\quad i = 1, \dots, d^2.
\]
This induces the partition
\[
\mathcal{P}_{\textrm{ang}}^d \coloneqq \{\Omega_{\textrm{ang}}^1, \dots, \Omega_{\textrm{ang}}^{d^2}\},
\]
illustrated in Fig.~\ref{fig:overview} \textbf{A} 3), consisting of $s = d^2$ angular splits. These settings are labeled as $\mathcal{P}_{\textrm{ang}}^2$, $\mathcal{P}_{\textrm{ang}}^3$, and $\mathcal{P}_{\textrm{ang}}^4$, respectively.

\noindent\textit{Irregular Angular Splitting.}
We further consider an irregular variant, where one projection angle is randomly masked within blocks of $s = d^2$ consecutive angles (Fig.~\ref{fig:overview} \textbf{A} 4)). In the experiments, the configurations follow $d \in \{2, 3, 4\}$.

\subsubsection{Proposed Multi-Partition Splitting (MPS)}

In contrast to single-partition approaches, one can also consider a family of multiple partitions $\mathcal{P}_i, i=1,\dots,P$ of the measurement index set that enforce simultaneous invariance along different geometric configurations.

\noindent\textit{Joint Angular and Detector Splitting ($\mathcal{P}_{\textrm{mult}}^d$):}
This strategy applies multi-partition masking ($P=2$) by introducing two separate, independent partition systems that operate along the angular and detector offset dimensions simultaneously. We define the primary angular partition component consisting of $s$ splits as:
\[
\mathcal P_{\textrm{ang}}^{s} = \{\Omega_{\textrm{ang}}^1, \dots, \Omega_{\textrm{ang}}^{s}\},
\]
where each subset $\Omega_{\textrm{ang}}^i$, $i = 1, \dots, s$ isolates projection angles uniformly distributed over $[m]$, i.e. 
\[
\Omega_{\textrm{ang}}^j = [m] \times \{ \nu \in [m] \mid \nu \bmod s = j \}, \quad j = 1, \dots, s.
\]
\noindent Similarly, we define the secondary detector partition component, also consisting of $s$ splits, as:
\[
\mathcal P_{\textrm{det}}^{s} = \{\Omega_{\textrm{det}}^1, \dots, \Omega_{\textrm{det}}^{s}\},
\]
where the coordinates are grouped element-wise via:
\[
\Omega_{\textrm{det}}^j = [m] \times \{ \nu \in [n] \mid \nu \bmod s = j \}, \quad j = 1, \dots, s.
\]
Each $\Omega_{\textrm{det}}^j$ selects vertical detector columns in a periodic and uniform manner (see Fig.~\ref{fig:overview} \textbf{A} 5)). 

Because this joint strategy incorporates both partition families simultaneously, the total self-supervised loss function aggregates their independent data-domain behaviors, evaluating $s = d^2$ splits within each geometric system:
\begin{align*}
\mathcal{L}_{\textrm{mult}}^d(\theta) = 
&\frac{1}{s}\sum_{\Omega \in \mathcal{P}_{\textrm{ang}}^{s}}
\| \mathbf{M}_{\Omega} A\left( f_\theta\left((B\circ g^{\Omega^c}) \mathbf{M}_{\Omega^c} y\right)\right) - \mathbf{M}_{\Omega} y \|_2^2 \\
+ &\frac{1}{s}\sum_{\Omega \in \mathcal{P}_{\textrm{det}}^{s}}
\| \mathbf{M}_{\Omega} A\left( f_\theta\left((B\circ g^{\Omega^c}) \mathbf{M}_{\Omega^c} y\right)\right) - \mathbf{M}_{\Omega} y \|_2^2.
\end{align*}
In the experiments, plots, and tables, this joint framework is denoted as $\mathcal{P}_{\textrm{mult}}^d$, meaning that the network is optimized under $s \in \{4, 9, 16\}$ angular splits and detector splits concurrently.

\noindent In the following, we show that training with multiple partitions enables the construction of a mapping that is invariant with respect to an induced partition $\hat{\mathcal{P}}$.

\subsection{Theoretical Analysis}
We demonstrate that multi-partition masking yields an inference function that is invariant with respect to the choice of a single-partition splitting mask. Furthermore, we establish a theorem characterizing the minimization error in the presence of noise that is correlated along the projection direction.
\begin{prop}[Invariance for multiple partitions]\label{assemble}
Let $\Omega = [m]\times[n]$ and let $
\mathcal P_k = \{\Omega_k^1,\dots,\Omega_k^{|\mathcal P_k|}\},
\qquad k=1,\dots,K,
$
be $K$ partitions of $\Omega$.
Consider the loss
\begin{align*}
\mathcal L^Y(\theta)
=
\sum_{k=1}^K
\sum_{i=1}^{|\mathcal P_k|}
\Big\|
\mathbf M_{\Omega_k^i}
\big(
A\!\circ\! f_\theta(\mathcal B\!\circ\! g^{\Omega_k^{i^c}})
(\mathbf M_{\Omega_k^{i^c}} y)
\big)
-
\mathbf M_{\Omega_k^i} y
\Big\|_2^2.
\end{align*}
Assume that $f_{\theta^\ast}$ minimizes $\mathcal L(\theta)$ (in expectation) under the conditional independence assumption of self-supervised masking, i.e. the noise in each subset $\Omega_k^i\in\mathcal{P}_k$ is independent from the noise in all other $\Omega_k^j\in\mathcal{P}_k, i\neq j$.
Then the assembled mapping $\hat{f}_{\theta^\ast}^\mathcal{P}\coloneqq\frac{1}{K}\sum_{k=1}^K f_{\theta^\ast}^{\mathcal{P}_k}$ is $\tilde{\mathcal{P}}$-invariant with 
\[
\tilde{\mathcal P}
=
\left\{
\bigcap_{k=1}^K \Omega_k^j
\;:\;
\Omega_k^j \in \mathcal P_k
\right\}.
\]
\end{prop}

\begin{proof}
    See supplementary.
\end{proof}
Proposition \ref{assemble} establishes that multi-partition training effectively enforces invariance with respect to the intersection of the underlying partitions, providing a theoretical explanation for its improved robustness under both correlated and uncorrelated noise observed in our experiments (see Section \ref{experiments}).
If we consider the case $K = 2$, where the two partitions correspond to angular and detector-wise splits ($\mathcal{P}_{\textrm{ang}}^s$ and $\mathcal{P}_{\textrm{det}}^s$) defined in Section~\ref{splitting}, Proposition \ref{assemble} shows that the partitioning of the invariant term in joint angular and detector splitting is identical to that in the pixel-wise regular masking scheme ($\mathcal{P}_{\textrm{lat}}^s$).

We now analyze the expected minimization error in the presence of scintillator-induced detector correlations modeled by a convolution with a one-dimensional kernel.

\begin{theorem}[Expected minimization error under angular splitting]\label{expected_error}
Let the forward model be
\begin{align}\label{forward_model_correlated}
    y = \mathcal A x + \eta,
\end{align}
where $\eta$ is zero-mean Gaussian noise which is independent across different projection angles but may be correlated within each angle.
Let $\mathcal P_{\textrm{ang}}$ be the angular partition defined above.
Then it holds that
\begin{align*}
&\sum_{\Omega_{\textrm{ang}} \in \mathcal P_{\textrm{ang}}}
\mathbb E
\Bigl\|
\mathbf M_{\Omega_{\textrm{ang}}}A
f_\theta\bigl((\mathcal B\circ g^{\Omega_{\textrm{ang}}^c})
\mathbf M_{\Omega_{\textrm{ang}}^c} y\bigr)
-
\mathbf M_{\Omega_{\textrm{ang}}} y
\Bigr\|_2^2
\\
&=
\sum_{\Omega_{\textrm{ang}} \in \mathcal P_{\textrm{ang}}}
\mathbb E
\Bigl\|
\mathbf M_{\Omega_{\textrm{ang}}}
A\circ f_\theta\bigl((B\circ g^{\Omega_{\textrm{ang}}^c})
\mathbf M_{\Omega_{\textrm{ang}}^c} y\bigr)
-
\mathbf M_{\Omega_{\textrm{ang}}} A x
\Bigr\|_2^2
\\
&\quad+
\sum_{\Omega_{\textrm{ang}} \in \mathcal P_{\textrm{ang}}}
\mathbb E
\|
\mathbf M_{\Omega_{\textrm{ang}}}(\eta)
\|_2^2 .
\end{align*}
\end{theorem}
\begin{proof}
    See supplementary.
\end{proof}
The theorem shows that under angular splitting the expected objective
decomposes into a reconstruction term and a noise energy term. This implies that, under angular splitting, intra-angle detector correlations contribute only through a noise-energy term that is independent of the network parameters. Consequently, the optimization objective is not directly affected by the detailed correlation structure of the noise.

\begin{figure*}
    \centering
    \includegraphics[width=0.9\linewidth]{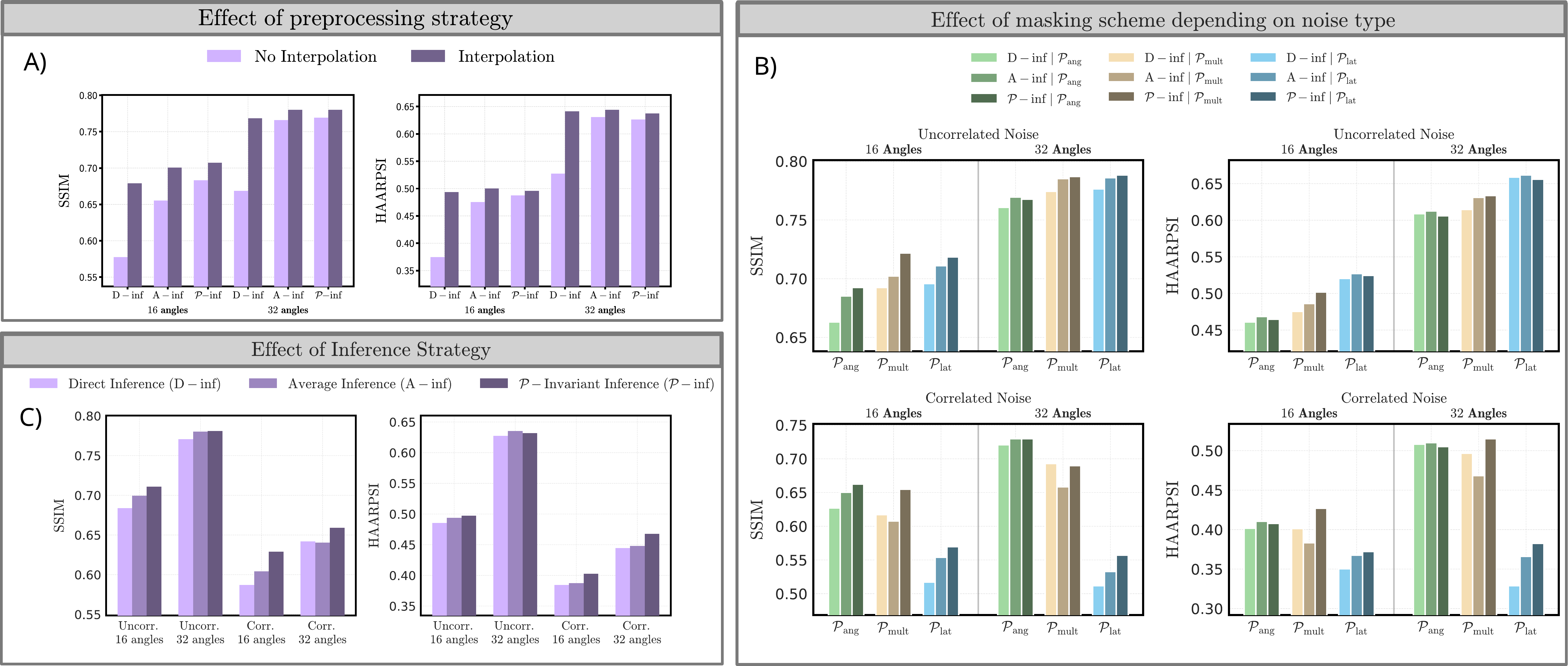} 
    \caption{\textbf{Summary of the main experimental findings on synthetic data.}
(A) Effect of preprocessing: interpolation of masked measurements prior to FBP improves SSIM and HaarPSI compared to zero-filling. 
(B) Influence of splitting strategy under uncorrelated and correlated noise regimes, evaluated in terms of SSIM and HaarPSI. 
(C) Comparison of inference strategies (D-inf, A-inf, and $\mathcal{P}$-inf).}
    \label{highlights_results}
\end{figure*}
\section{Experiments}\label{experiments}

We evaluate all methods on the test set of the LoDoPaB-CT dataset~\cite{leuschner2021lodopab}, a standard benchmark for sparse-view CT reconstruction, and on the real-world 2DeteCT dataset \cite{kiss20232detect}. For LoDoPaB-CT, the dataset is split into 1000 training images, 184 validation images, and 400 test images. Reconstructions are performed using the same U-Net architecture as in~\cite{gruber2024sparse2inverse} and \cite{hendriksen2020noise2inverse}. The provided reconstructions are rescaled to a spatial resolution of $336 \times 336$ and normalized such that their intensity values lie in the range $[0,1]$. The projection geometry consists of uniformly distributed angles in $[0,\pi]$. All images are multiplied by a circular mask, and all evaluation metrics are computed within the masked region.

To assess performance under different noise characteristics, we consider both uncorrelated and correlated measurement noise on LoDoPaB-CT. For the uncorrelated setting, projection data are corrupted with Poisson noise corresponding to photon counts of 3500 and 6000, following the procedure of~\cite{hendriksen2020noise2inverse}. To simulate correlated noise arising from scintillator blur, we add Gaussian noise with standard deviation $\sigma=0.1$ and subsequently convolve it along the detector direction using a kernel with bandwidth $l=5$.

For the 2DeteCT dataset, we use the mode-1 high-noise acquisitions. The data are split into 100 training, 50 validation, and 100 test slices. Each original scan contains 3600 projection angles and 1912 detector bins. For our sparse-view reconstruction setting, we downsample the detector direction by a factor of two, resulting in 956 detector bins, and uniformly subsample the angular measurements to 100 projection views.

Forward projections and filtered backprojection are implemented using tomosipo~\cite{hendriksen-2021-tomos} and LION \footnote{https://github.com/CambridgeCIA/LION}, respectively. All neural networks are implemented and trained in PyTorch \cite{paszke2019pytorch}. We use the Adam optimizer \cite{kingma2014adam} with a learning rate of $10^{-4}$ and mixed-precision training via PyTorch’s \texttt{GradScaler}.
During training, model checkpoints are evaluated on the validation set using all three inference strategies independently. For each strategy, we track and select the specific network weights that yield the highest validation SSIM, resulting in three distinct optimized weight configurations (one for D-inf, one for A-inf, and one for $\mathcal{P}$-inf). These individual weight sets are then used for the final evaluations on the test set.  We observe that lattice-based splitting ($\mathcal{P}_{\textrm{lat}}$) converges faster than angular and joint splitting schemes. Accordingly, lattice-based models are trained for 6000 epochs, while angular ($\mathcal{P}_{\textrm{ang}}$) and joint angular–detector ($\mathcal{P}_{\textrm{mult}}$) models are trained for 10{,}000 epochs.
A description of all evaluated configurations and used abbreviations is provided in Table~\ref{abbreviations}.
The source code, data, and experiment configurations are publicly available at the project repository 
(\href{https://github.com/Nadja1611/Design-Choices-in-Splitting-Based-Self-Supervised-Sparse-View-CT-Reconstruction}{https://github.com/Nadja1611/Design-Choices-in-Splitting-Based-Self-Supervised-Sparse-View-CT-Reconstruction}).

\subsection*{Evaluation metrics}
We report Peak Signal-to-Noise Ratio (PSNR), Structural Similarity Index Measure (SSIM), Learned Perceptual Image Patch Similarity (LPIPS)~\cite{zhang2018unreasonable}, and HaarPSI~\cite{reisenhofer2018haar}. While PSNR and SSIM are widely used, they do not always capture perceptual or structurally relevant differences in reconstruction quality. We therefore additionally report LPIPS and HaarPSI, which provide complementary perceptual and structural measures based on deep feature representations and localized wavelet analysis, respectively. Lower LPIPS values indicate higher perceptual similarity, while higher HaarPSI values indicate better structural fidelity.
For the 2DeteCT experiments, the reference images were obtained from reconstructions of the full mode-2 data. To compensate for possible global intensity-scale differences between these reference reconstructions and the sparse-view predictions, we applied a scalar calibration before metric computation. For each test slice, a multiplicative factor was selected from a fixed range $[0.3, 2.5]$ by maximizing the PSNR with respect to the reference image. All reported quantitative metrics on 2DeteCT were then computed using the corresponding scaled reconstruction.

Having established the experimental framework, we now systematically isolate and analyze the individual components of our unified splitting architecture introduced in Section~\ref{sec:method}. 

\begin{table*}[t]
\centering
\small
\begin{tabularx}{\textwidth}{ll>{\raggedright\arraybackslash}X}
\toprule\toprule
\textbf{Method} & \textbf{Splitting} & \textbf{Configuration and total splits $s$} \\
\midrule

$\mathcal{P}_\textrm{lat}^d$ &
Lattice &
$d \in \{2,3,4\}$, with $s=d^2 \in \{4,9,16\}$.
One pixel is selected per $2\times2$, $3\times3$, or $4\times4$ patch. \\

$\mathcal{P}_\textrm{ang}^d$ &
Angular &
$d \in \{2,3,4\}$, with $s=d^2 \in \{4,9,16\}$.
One projection angle is masked in each block of 4, 9, or 16 consecutive angles. \\

$\mathcal{P}_\textrm{mult}^d$ &
Angular--detector &
$d \in \{2,3,4\}$, with $s=d^2 \in \{4,9,16\}$ per direction.
One projection angle and one detector row are masked in each block of 4, 9, or 16. \\

\bottomrule
\end{tabularx}
\vspace{0.1cm}
\caption{Overview of the abbreviations and partition indices used for the benchmarked splitting approaches.}
\label{abbreviations}
\end{table*}
\section*{Synthetic LoDoPaB-CT results}
\subsection{Effect of the preprocessing strategy}
As described in Section~\ref{sec:preprocessing}, preprocessing either sets masked values to zero or replaces them via linear interpolation before FBP, producing the network input. For lattice masking schemes, this is given by
\[
\mathbf{M}_\Omega (y \ast \kappa) + \mathbf{M}_{\Omega^c} y,
\]
with convolution kernels 
\[
\kappa_\textrm{lat} = \frac{1}{4}
\begin{bmatrix}
0 & 1 & 0 \\
1 & 0 & 1 \\
0 & 1 & 0
\end{bmatrix}, \quad
\kappa_{\textrm{ang}} = \frac{1}{2}[1 \ 0 \ 1], \quad
\kappa_{\textrm{det}} = \frac{1}{2}[1 \ 0 \ 1]^T.
\]

For the multi-partition splitting scheme, only the interpolation-based preprocessing is considered; the zero-masking variant is not used. Masked detector rows and projection angles are replaced by linear interpolation using $\kappa_\mathrm{det}$ and $\kappa_\mathrm{ang}$, respectively.

Figure~\ref{highlights_results}A) shows that interpolation prior to FBP consistently improves reconstruction quality, particularly for direct inference (D-inf) settings. While intuitive, this is noteworthy because the original Sparse2Inverse approach~\cite{gruber2024sparse2inverse} does not employ interpolation, leaving masked values 0. This effect is visually apparent in Figure~\ref{interpolation} (supplementary), where omitting interpolation leads to severe undersampling artifacts. 

To provide a consolidated view of this architectural component in Figure~\ref{highlights_results} A), the reported SSIM and HaarPSI scores are obtained using the uncorrelated Poisson noise regime at a photon count of 6000, averaged across all evaluated splitting configurations for both $s=4$ and $s=9$ total splits. The exact quantitative metrics across all individual parameters are documented in Table~\ref{tab:uncorrelated_6000} in the supplementary material. We observe that while direct inference benefits most drastically from preprocessing, a visible performance gain is shared across all evaluated splitting architectures and inference strategies.

\subsection{Influence of splitting scheme}
The impact of the splitting strategy differs substantially between the uncorrelated and correlated noise regimes.

\paragraph{Influence of splitting scheme in the Poisson (uncorrelated) noise case}
In the uncorrelated noise setting, we compare the masking strategies $\mathcal{P}_{\mathrm{lat}}$ and $\mathcal{P}_{\mathrm{ang}}$ and observe that lattice-based masking $\mathcal{P}_{\mathrm{lat}}$ consistently outperforms angular masking across all evaluated configurations. This trend is illustrated in Fig.~\ref{highlights_results}(B, top row), which reports the average SSIM and HaarPSI over the three splitting factors (4, 9, and 16). Corresponding PSNR and LPIPS results are shown in Fig.~\ref{lpips_psnr_split} (top row).

Among the considered metrics, LPIPS indicates a slight advantage of multipartition splitting over both lattice and angular masking in the 16-split setting. In contrast, SSIM shows comparable performance between multipartition and lattice splitting, while HaarPSI consistently favors lattice splitting. Since HaarPSI is designed to capture perceptual structural similarity, these results suggest that lattice splitting provides the most reliable performance in the uncorrelated noise regime. 

Interestingly, all evaluated metrics consistently indicate that multipartition splitting ($\mathcal{P}_{\textrm{mult}}$) yields a systematic improvement over pure angular masking ($\mathcal{P}_{\textrm{ang}}$). This finding is particularly noteworthy as it demonstrates that augmenting the angular splitting mechanism with concurrent detector-wise masking provides a powerful regularizing effect, successfully exploiting additional structural measurement redundancies while preserving the ability to reduce sparse-view artifacts.

These observations are further supported by the visual reconstructions in Fig.~\ref{reco_32_6000_uncorr} (32 angles, photon count 6000) and in Figs.~\ref{reco_16_6000_uncorr}, \ref{reco_16_3500_uncorr}, and \ref{reco_32_3500_uncorr} in the supplementary material. Detailed quantitative results are reported in Tables~\ref{tab:uncorrelated_3500} and~\ref{tab:uncorrelated_6000} in the supplementary.

\begin{figure*}
    \centering
    \includegraphics[width=0.99\linewidth]{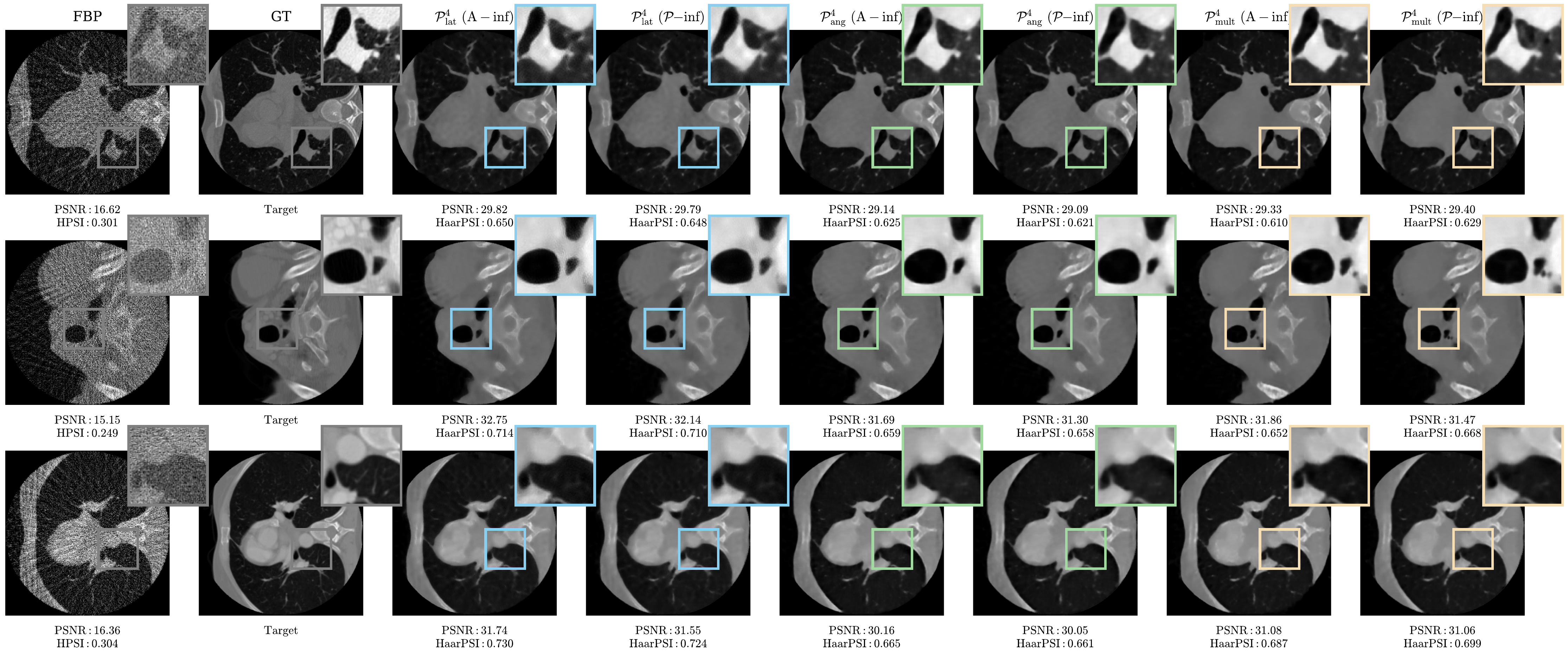}
\caption{Reconstruction results for three examples from the LoDoPaB-CT dataset using 32 projection angles and Poisson noise with a photon count of 6000. For each masking scheme, the displayed reconstructions correspond to the best-performing configuration reported in Table~\ref{tab:uncorrelated_6000}. Columns 3-4 show results for $\mathcal{P}{\textrm{lat}}^4$ (irregular splitting), columns 5-6 for $\mathcal{P}{\textrm{ang}}^3$ (irregular splititng), and columns 7--8 for $\mathcal{P}_{\textrm{mult}}^2$. All configurations use irregular masking during both training and inference. The best-performing inference strategy, either A-inf or $\mathcal{P}$-inf depending on the configuration, is shown.}
    \label{reco_32_6000_uncorr}
\end{figure*}

\paragraph{Influence of splitting scheme in the correlated noise case}

In the correlated noise regime, the relative behavior of the splitting schemes changes substantially. Angular splitting clearly outperforms lattice splitting across all metrics. The multipartition strategy further improves performance over lattice masking and, in the highly sparse setting with 16 angles, even surpasses angular splitting under the invariant inference configuration (Fig.~\ref{highlights_results}(B, bottom row), Fig.~\ref{lpips_psnr_split} supplementary (bottom row)).

These findings are consistent with the qualitative reconstructions in Fig.~\ref{reco_32_corr}, where lattice-based reconstructions remain visibly contaminated by noise, whereas both angular and multipartition splitting effectively suppress structured noise artifacts and yield cleaner reconstructions (see Table \ref{tab:correlated} in supplementary).

\subsection{Influence of the number of splits}
We evaluate three splitting configurations corresponding to $s \in \{4, 9, 16\}$ number of splits. As summarized in Table \ref{abbreviations}, for lattice masking, these parameters correspond to a patch size of $2\times2$, $3\times3$, and $4\times4$ pixels, respectively. For angular and multi-partition splitting, this setup corresponds to retaining one out of every 4, 9, or 16 consecutive projections (and detector rows).

\paragraph{Number of splits in the Poisson noise case}
Figure~\ref{gridsizes} (top row) illustrates the effect of the number of splits on all four performance metrics under $\mathcal{P}$-inf inference. Across all splitting frameworks, configurations utilizing 9 and 16 partitions generally achieve a slightly higher reconstruction performance than the coarser 4-partition setting. Overall, however, the direct influence of the masking density on reconstruction quality remains minor. This stable behavior is highly consistent across all evaluation metrics and experimental sub-settings. Detailed numerical results are provided in Tables~\ref{tab:uncorrelated_3500} and~\ref{tab:uncorrelated_6000} of the supplementary material.

\paragraph{Number of splits in the correlated noise case}
In the correlated noise regime, a distinct trend is observed. As shown in Fig.~\ref{gridsizes} (bottom row), reducing the number of splits (i.e., selecting $s=4$) yields a slight performance improvement for both lattice and multipartition masking. Conversely, for angular splitting, the reconstruction metrics remain largely stationary across the different numbers of splits, demonstrating that this strategy is highly robust to parameter variations. This insensitivity is expected, as the angular decoupling is inherently designed to bypass horizontal scintillator blur effects regardless of the number of splits. Quantitative results are documented in Table~\ref{tab:correlated} of the supplementary material.

\subsection{Influence of inference strategy:}
We perform our experiments using the three inference strategies
D-inf, A-inf and $\mathcal{P}$-inf as described in Section \ref{sec:inference}. For $\mathcal{P}$-inf we apply the full forward operator to the network output. In this case, the number of projection angles is $M = n \cdot \pi/2$, where $n$ denotes the image height and width (336 in our case), and $m$ is the number of projection angles available in the sparse CT problem. 

\paragraph{Influence of inference strategy in the Poisson noise case}
The benefits of A-inf and $\mathcal{P}$-inf over D-inf across all splitting strategies are illustrated in Fig.~\ref{highlights_results} (C, green line). Both SSIM and HaarPSI indicate a consistent, albeit modest, improvement when comparing direct inference to average and $\mathcal{P}$-invariant inference.
In this figure, results are averaged over all three splitting strategies and all three numbers of splits (4, 9, and 16), providing an overall view of the impact of the inference strategy. A more detailed breakdown in Tables~\ref{tab:uncorrelated_3500} and~\ref{tab:uncorrelated_6000} in the supplementary material shows that this trend is consistent across all splitting schemes and all numbers of splits.
This effect is particularly relevant for lattice masking, as its original formulation in~\cite{unal2021self} was introduced specifically in combination with D-inf inference.

\paragraph{Influence of inference strategy in the correlated noise case}
In the correlated noise regime, similar trends are observed, with a more pronounced effect than in the Poisson noise case (Fig.~\ref{highlights_results}(C, brown)). In particular, invariant inference consistently yields improved performance compared to A-inf, and more significantly compared to D-inf.
Detailed results for each splitting scheme are provided in Table~\ref{tab:correlated} in the supplementary material. The improvement is most pronounced for the lattice masking strategy.

\subsubsection{Regular vs. irregular splitting}
To determine the configurations for our evaluation of irregular splitting, we first identify the optimal number of splits for each individual partitioning method based on the highest mean validation SSIM (averaged over the three inference strategies); irregular training masks are then deployed specifically for these top-performing splitting factors. We evaluate whether these \emph{irregular masking schemes} can improve reconstruction quality, drawing on the well-established role of randomized partitioning as a regularization mechanism to enhance generalization performance in machine learning. To the best of our knowledge, however, irregular training masks have not yet been explored in the context of Proj2Proj~\cite{unal2021self}, Noise2Inverse~\cite{hendriksen2020noise2inverse}, or Sparse2Inverse~\cite{gruber2024sparse2inverse}. 

For each splitting-based method, we compare two distinct configurations:
\begin{itemize}
\item regular splitting during both training and inference,
\item irregular splitting during training paired with regular splitting during inference.
\end{itemize}

\paragraph{Effect of irregular splitting in the Poisson noise case}
The results are shown in Fig.~\ref{metrics_structured_vs_unstructured} (top row) of the supplementary material. In this experiment, performance metrics are directly compared between the best-performing regular splitting variants—specifically lattice ($\mathcal{P}_{\textrm{lat}}^2$ and $\mathcal{P}_{\textrm{lat}}^4$ for 16 and 32 angles, respectively) and angular ($\mathcal{P}_{\textrm{ang}}^3$ and $\mathcal{P}_{\textrm{ang}}^4$ for 16 and 32 angles, respectively)—and their respective irregular counterparts.

We observe a small but consistent performance improvement when irregular masks are deployed during the training phase. While the absolute numerical gains are modest, they remain consistently present across all three evaluated inference strategies. This systematic trend suggests that introducing stochastic randomness into the sinogram partitioning process acts as an effective regularizer, noticeably improving the generalization capability and robustness of the learned reconstruction mapping. The corresponding detailed quantitative results are documented in Tables~\ref{tab:uncorrelated_3500} and~\ref{tab:uncorrelated_6000} of the supplementary material.

\paragraph{Effect of irregular splitting in the correlated noise case}
Similar trends are observed in the correlated noise regime. Specifically, the results shown in Fig.~\ref{metrics_structured_vs_unstructured} (bottom row) compare the best-performing configurations under correlated noise: lattice ($\mathcal{P}_{\textrm{lat}}^2$ for both 16 and 32 angles) versus angular masking ($\mathcal{P}_{\textrm{ang}}^2$ and $\mathcal{P}_{\textrm{ang}}^3$ for 16 and 32 angles, respectively) against their irregular variants. 

In this scenario, the regularizing impact of irregular splitting is even more pronounced, yielding a larger relative improvement over structured splitting compared to the Poisson noise case. The corresponding quantitative results are reported in Table~\ref{tab:correlated} of the supplementary material, where the performance boost provided by stochastically randomized masks is most evident within the lattice masking framework.

\section*{Real-world evaluation on 2DeteCT}

As we figured out that the number of splits has no relevant influence on the reconstruction results on the synthetic data, we set $d=3$ on the 2DeteCT dataset.
The quantitative results summarized in Figure~\ref{highlights_results_2detect} and Table~\ref{tab:splitting_results_2detect} in the supplementary show that several key trends from the synthetic experiments also transfer to real measured data. In particular, averaging-based inference and $\mathcal{P}$-invariant inference generally improve over direct inference, especially in terms of PSNR and HaarPSI. This indicates that exploiting multiple masked reconstructions at test time remains beneficial also under realistic acquisition conditions. Moreover, angular splitting achieves the best PSNR and SSIM values, with the irregular interpolated configuration $\mathcal{P}_{\mathrm{ang}}^{3}$ obtaining the strongest overall performance in these metrics. This supports the observation that angular splitting is well suited to realistic noisy measurements, where local detector-neighborhood dependencies may reduce the effectiveness of lattice-based masking.

The effect of interpolation is less uniform than in the synthetic LoDoPaB-CT experiments. While interpolation improves several configurations and confirms the general preprocessing trend observed in simulation, zero-filled inputs remain competitive in some cases and can yield better values for individual metrics such as SSIM or LPIPS. This reduced benefit may be related to physical effects present in the 2DeteCT measurements that are not captured by the synthetic noise model, or to the less severely undersampled acquisition setting with 100 projection views. Therefore, the interpolation results on 2DeteCT should be interpreted with caution.

The representative reconstructions in Figure~\ref{reco_2detect} provide a visual interpretation of these quantitative findings. Sparse2Inverse removes noise effectively, but tends to oversmooth the reconstructions, leading to a loss of fine structural detail. In contrast, lattice and multi-partition splitting show visually similar behaviour: both suppress noise less aggressively, but preserve local image structures more strongly. This detail preservation may explain why lattice and multi-partition configurations achieve competitive or higher HaarPSI values in some cases, despite not always obtaining the best PSNR or SSIM. Overall, the visual results confirm that different splitting strategies induce different trade-offs between denoising strength and structural detail preservation on real measured CT data.

\begin{figure*}
    \centering
    \includegraphics[width=0.99\linewidth]{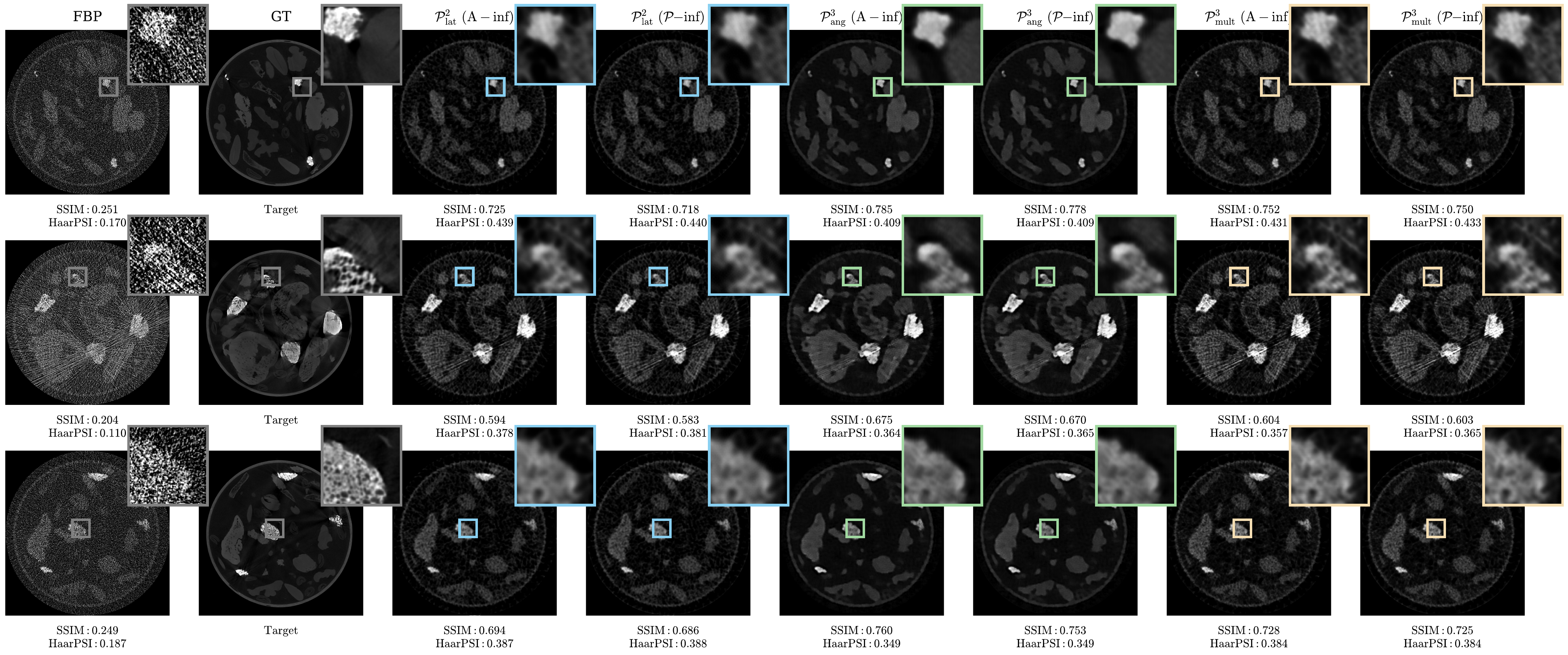}
\caption{\textbf{Reconstruction results on three examples of the 2DeteCT dataset for 100 projection angles in the mode-1 high-noise setting.}
For each splitting scheme, the shown reconstructions correspond to the best-performing configuration highlighted in Table~\ref{tab:splitting_results_2detect} in the supplementary. columns 2-3 show results for $\mathcal{P}_{\textrm{lat}}^3$ with regular splitting, columns 4-5 show $\mathcal{P}_{\textrm{ang}}^3$ with irregular splitting, and the last two columns show $\mathcal{P}_{\textrm{mult}}^3$. For each splitting strategy, the two displayed reconstructions correspond to the best-performing inference strategies, A-inf and $\mathcal{P}$-inf.}
    \label{reco_2detect}
\end{figure*}

\begin{figure}
    \centering
    \includegraphics[width=0.99\linewidth]{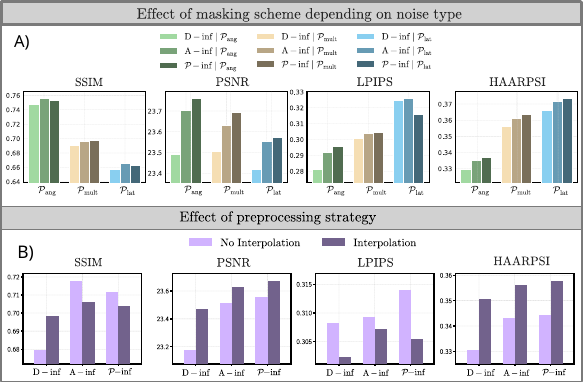} 
    \caption{\textbf{Summary of the main experimental findings on 2DeteCT dataset.}
(A) Influence of splitting strategy on real 2DeteCT mode-1 high-noise data, evaluated in terms of SSIM, PSNR, LPIPS and HaarPSI. (B) Effect of preprocessing: interpolation of masked measurements prior to FBP improves PSNR, LPIPS and HaarPSI, while in terms of SSIM the interpolation is only helpful for direct inference (note that here the average across all three splitting approaches is depicted). }
    \label{highlights_results_2detect}
\end{figure}

\section{Discussion}\label{sec:discussion}
We systematically benchmarked splitting-based self-supervised reconstruction methods for sparse-view CT, isolating the distinct effects of preprocessing, inference strategy, splitting design, number of splits, and noise
characteristics. Experiments on synthetic LoDoPaB-CT data allow controlled analysis under independent and correlated noise models, while the real
2DeteCT experiments assess whether the resulting trends transfer to measured
CT data. To ensure a comprehensive assessment, we evaluate reconstruction
quality using a set of complementary image quality metrics. Following recent recommendations in medical imaging computing, we complement PSNR and SSIM with perceptual and structure-sensitive metrics such as LPIPS and HaarPSI.
These metrics provide additional perspectives on reconstruction quality, particularly for artifacts and structural differences that may not be fully reflected by error-based scores alone. 

\paragraph{Choice of Splitting Strategy}
Our results show that the optimal splitting strategy depends fundamentally on
the underlying noise statistics of the measurements. For spatially
uncorrelated Poisson noise, lattice splitting consistently achieves the best
structural and perceptual performance in terms of HaarPSI, alongside
competitive results in SSIM and LPIPS. This demonstrates that local spatial
prediction, or pixel-wise conditioning, is well aligned with independent noise
assumptions and can effectively exploit local sinogram redundancy.

In contrast, under correlated noise along the detector array, lattice
splitting no longer provides the best performance, since neighboring masked
and unmasked detector elements remain statistically coupled. In this regime,
angular splitting becomes clearly superior. By withholding entire projection
angles, angular partitioning reduces the influence of detector-direction
correlations on the self-supervised loss and better preserves the independence
assumptions required by data splitting. In extremely sparse settings with 16
projection views, our proposed multi-partition splitting
($\mathcal{P}_{\textrm{mult}}^s$) provides additional structural gains,
indicating that combining angular and detector-direction partitioning can be
beneficial when measurement density is severely limited.

The real 2DeteCT experiments support this general trend, but also show that
the transfer from synthetic noise models to measured data is not exact.
Angular splitting achieves the strongest overall PSNR and SSIM on 2DeteCT,
consistent with the observation that view-wise splitting is more robust when
local detector-neighborhood dependencies are present. However, lattice and
multi-partition splitting remain competitive in HaarPSI for some
configurations, suggesting that they may preserve local structural detail
better even when they do not achieve the best error-based scores.

\paragraph{Choice of Inference Strategy}
Across all evaluated settings, the choice of inference strategy reveals a
consistent trend: both averaging-based inference (A-inf) and our proposed
$\mathcal{P}$-invariant inference ($\mathcal{P}$-inf) outperform standard
direct inference (D-inf). The performance gains are stable across noise
regimes, splitting types, sparsity levels, and also transfer to the real
2DeteCT experiments. This indicates that aggregating multiple masked
predictions and preserving the $\mathcal{P}$-invariance structure imposed
during training are beneficial compared with simple single-pass network
testing. In practice, D-inf should therefore be avoided when computational
resources permit, while $\mathcal{P}$-inf provides the most principled
inference strategy due to its closer alignment with the training objective.

\paragraph{Choice of Preprocessing}
Linear interpolation of masked sinogram elements prior to filtered
backprojection consistently improves reconstruction quality on the synthetic
LoDoPaB-CT experiments. This effect is particularly pronounced for angular
splitting and in highly undersampled settings, where zero-filled sinograms
produce severe FBP artifacts before the data are passed to the network.
Notably, this challenges earlier self-supervised splitting approaches such as
Sparse2Inverse~\cite{gruber2024sparse2inverse}, which rely on raw zero-filling
or data restriction. Our findings indicate that simple directional
interpolation can provide a more continuous and informative network input by
reducing severe undersampling artifacts before they enter the learned
reconstruction space.

On the real 2DeteCT data, however, the effect of preprocessing is more
method-dependent. As shown in Table~\ref{tab:splitting_results_2detect}, interpolation improves
the angular splitting configuration, but does not yield a consistent gain for
lattice-based Proj2Proj-style splitting or for the multi-partition strategy.
This suggests that interpolation is most beneficial when the mask removes
entire projection angles, whereas for lattice and multi-partition masks the
missing samples are distributed more locally and zero-filling can remain
competitive. The weaker and less uniform interpolation effect on 2DeteCT also
indicates that preprocessing choices may depend on real acquisition effects
that are not fully captured by the synthetic noise models.

\paragraph{Choice of Splitting Density}
The total number of splits $s$ exerts a secondary, comparatively limited impact on overall performance. In Poisson noise settings, higher partitioning densities ($d = 3, 9$ splits and $d=4, 16$ splits) yield minor improvements. In correlated noise regimes, however, lattice-based splitting schemes show a slight preference for lower densities ($s=4, 9$ splits). Overall, these results suggest that splitting density plays a secondary role compared with the choice of splitting geometry and its compatibility with the noise structure.

\paragraph{Summary: Practical Selection Rule}
Based on the compilation of our experimental evidence, we establish the following configuration guidelines for self-supervised sparse-view CT reconstruction:
\begin{itemize}
    \item \textbf{Under Uncorrelated Noise:} Irregular lattice splitting ($\mathcal{P}_{\textrm{lat}}^s$) is recommended to effectively exploit dense local pixel neighborhoods.
    \item \textbf{Under synthetic correlated noise and on real 2DeteCT data:} Angular splitting ($\mathcal{P}_{\textrm{ang}}^d$) reduces the effect of detector-direction correlations on the self-supervised objective. For ultra-sparse view regimes (synthetic data with 16 angles), joint multi-partition splitting ($\mathcal{P}_{\textrm{mult}}^d$) provides the highest structural fidelity.
    \item \textbf{Multi-Partition Splitting:} In scenarios where angular-based frameworks are utilized under uncorrelated noise conditions, transitioning to the multi-partition variant ($\mathcal{P}_{\textrm{mult}}^d$) is beneficial, as the concurrent detector-direction masking systematically improves upon the performance of pure angular splitting.
\item \textbf{Preprocessing:} On synthetic datasets with highly sparse sampling (16--32 projection angles), linear interpolation prior to FBP consistently improves reconstruction quality compared to zero-filling and restriction masks. On the real 2DeteCT dataset, this trend is less pronounced, with lattice and multi-partition splitting showing comparable or slightly better metrics without interpolation. This discrepancy may be attributable to the higher number of projections (100 vs. 16 and 32), although other real-world acquisition characteristics may also contribute.
    \item \textbf{Inference:} Ensembling pipelines (A-inf or $\mathcal{P}$-inf) are superior to direct inference (D-inf) across all settings as they stabilize structural predictions and enforce strict measurement consistency.
    \item \textit{Number of Splits:} Setting the number of splits $s$ between $9 \le s \le 16$ is optimal for lattice/Poisson configurations, while the split density acts as a low-priority tuning parameter in angular noise regimes.
\end{itemize}
\paragraph{Limitations}
Two main limitations remain. First, all experiments are restricted to
two-dimensional CT reconstruction. Although we evaluate both controlled
parallel-beam simulations on LoDoPaB-CT and real fan-beam measurements from
the 2DeteCT dataset, the proposed design guidelines may not transfer
directly to fully three-dimensional CT. In 3D, cone-beam geometry, helical
trajectories, larger detector correlations, and substantially different
memory constraints may affect the relative behavior of partitioning,
preprocessing, and inference strategies.

Second, model selection in our benchmark relies on reference-based validation
metrics, in particular validation SSIM. This is appropriate for a controlled
comparison of design choices, but it is not fully self-supervised in the
strict deployment sense, since such reference images are typically not
available in practical sparse-view CT settings. Future work should therefore
investigate ground-truth-free stopping criteria and model-selection
strategies, for example based on held-out measurement consistency,
self-supervised validation losses, or stability across complementary data
splits.

\paragraph{Outlook}
Future research directions will focus on translating this unified framework onto more real-world clinical and industrial CT datasets to confirm our design guidelines under complex, un-modeled hardware imperfections and polychromatic beam effects. Furthermore, we intend to explore a wider continuum of angular sampling densities to chart exactly where the benefits of self-supervised decoupling saturate.
Another promising frontier is the development of fully adaptive or learned splitting mechanisms that dynamically discover optimal partitioning boundaries directly from the noise print of the sinogram, alongside self-contained stopping criteria that operate independent of ground-truth validation sets. Finally, porting these deconstructed splitting modules to adjacent ill-posed inverse problems, such as sparse accelerated Magnetic Resonance Imaging (MRI) and Positron Emission Tomography (PET), represents a natural extension of this work.


\bibliographystyle{plain}
\bibliography{main}

\newpage
\onecolumn
\section{Supplement}

\begin{proof}[Proof of Proposition \ref{assemble}]
First note that the functions $f_{\theta^\ast}^{\mathcal{P}_k}$ are $\mathcal P_k$-invariant via \eqref{eq:assembly}.  
Now consider an element of the intersection
\[
\tilde{\Omega}
=
\bigcap_{k=1}^K \Omega_k^j,
\qquad
\Omega_k^j \in \mathcal P_k.
\]
By construction,
\[
\tilde{\Omega} \subseteq \Omega_k^j
\quad \text{for all } k=1,\dots,K.
\]
Suppose, for contradiction, that the prediction of $\hat{f}_{\theta^\ast}^\mathcal{P}$ on $\tilde{\Omega}$ depends on values of $y$ inside $\tilde{\Omega}$.  
Since $\tilde{\Omega} \subseteq \Omega_k$ for every $k$, this would imply that for some $k$ the prediction on $\Omega_k$ with $f_{\theta^\ast}^{\mathcal{P}_k}$ depends on values of $y$ inside $\Omega_k$, contradicting $\mathcal P_k$-invariance of $f_{\theta^\ast}^{\mathcal{P}_k}$. Hence $\hat{f}_{\theta^\ast}^\mathcal{P}$ is invariant with respect to the partition $\tilde{\mathcal P}$.
\end{proof}

\begin{proof}[Proof of Theorem \ref{expected_error}]
As a first step, we fix $\Omega_{\text{ang}} \in \mathcal P_{\text{ang}}$ and define
\[
y_{\Omega_{\text{ang}}}
:=
\mathbf M_{\Omega_{\text{ang}}}
A\circ f_\theta\bigl(( B\circ g)^{\Omega_{\text{ang}}^c}
\mathbf M_{\Omega_{\text{ang}}^c} y\bigr).
\]
Using \eqref{forward_model_correlated} we get
\[
\mathbf M_{\Omega_{\text{ang}}} y
=
\mathbf M_{\Omega_{\text{ang}}} A x
+
\mathbf M_{\Omega_{\text{ang}}}(\eta),
\]
and hence,
\[
y_{\Omega_{\text{ang}}}
-
\mathbf M_{\Omega_{\text{ang}}} y
=
y_{\Omega_{\text{ang}}}
-
\mathbf M_{\Omega_{\text{ang}}} A x
-
\mathbf M_{\Omega_{\text{ang}}}(\eta).
\]
Expanding the squared norm one obtains
\begin{align*}
&
\bigl\|
y_{\Omega_{\text{ang}}}
-
\mathbf M_{\Omega_{\text{ang}}} y
\bigr\|_2^2
\\
&=
\bigl\|
y_{\Omega_{\text{ang}}}
-
\mathbf M_{\Omega_{\text{ang}}}  A x
\bigr\|_2^2
+
\bigl\|
\mathbf M_{\Omega_{\text{ang}}}(\eta)
\bigr\|_2^2
\\
&\quad
-
2
\Bigl\langle
y_{\Omega_{\text{ang}}}
-
\mathbf M_{\Omega_{\text{ang}}} A x,
\mathbf M_{\Omega_{\text{ang}}}(\eta)
\Bigr\rangle.
\end{align*}
Taking expectations yields
\begin{align*}
&
\mathbb E
\bigl\|
y_{\Omega_{\text{ang}}}
-
\mathbf M_{\Omega_{\text{ang}}} y
\bigr\|_2^2
\\
&=
\mathbb E
\bigl\|
y_{\Omega_{\text{ang}}}
-
\mathbf M_{\Omega_{\text{ang}}}  A x
\bigr\|_2^2
+
\mathbb E
\bigl\|
\mathbf M_{\Omega_{\text{ang}}}(\eta)
\bigr\|_2^2
\\
&\quad
-
2\,
\mathbb E
\Bigl\langle
y_{\Omega_{\text{ang}}}
-
\mathbf M_{\Omega_{\text{ang}}} A x,
\mathbf M_{\Omega_{\text{ang}}}(\eta)
\Bigr\rangle.
\end{align*}
It remains to show that the cross term vanishes.

By construction, 
$y_{\Omega_{\text{ang}}}$ depends only on
$\mathbf M_{\Omega_{\text{ang}}^c} y$.
Since angular splitting masks entire projection angles and noise correlations are restricted to individual angles, it holds that
\[
\mathbf M_{\Omega_{\text{ang}}}(\eta)
\quad\text{is independent of}\quad
\mathbf M_{\Omega_{\text{ang}}^c}(\eta).
\]
From this we conclude that 
$y_{\Omega_{\text{ang}}}$ is independent of
$\mathbf M_{\Omega_{\text{ang}}}(\eta)$.
Because $\eta$ has zero mean the cross term vanishes,
\[
\mathbb E
\Bigl\langle
y_{\Omega_{\text{ang}}}
-
\mathbf M_{\Omega_{\text{ang}}} A x,
\mathbf M_{\Omega_{\text{ang}}}(\eta)
\Bigr\rangle
=0.
\]
We obtain
\begin{align*}
\mathbb E
\bigl\|
y_{\Omega_{\text{ang}}}
-
\mathbf M_{\Omega_{\text{ang}}} y
\bigr\|_2^2
=
\mathbb E
\bigl\|
y_{\Omega_{\text{ang}}}
-
\mathbf M_{\Omega_{\text{ang}}}  A x
\bigr\|_2^2
+
\mathbb E
\bigl\|
\mathbf M_{\Omega_{\text{ang}}}(\eta)
\bigr\|_2^2,
\end{align*}
and summing over all $\Omega_{\text{ang}} \in \mathcal P_{\text{ang}}$
yields the claim.
\end{proof}

\begin{table*}[ht]
\centering
\small
\setlength{\tabcolsep}{3.5pt}
\resizebox{\textwidth}{!}{
\begin{tabular}{lcc|ccc|ccc|ccc|ccc}
\toprule\toprule
& & & \multicolumn{12}{c}{\textbf{Performance metrics}} \\
\cmidrule(lr){4-15}
\textbf{Split} & \textbf{Irregular} & \textbf{Interp.} & \multicolumn{3}{c}{\textbf{PSNR}}& \multicolumn{3}{c}{\textbf{SSIM}}& \multicolumn{3}{c}{\textbf{LPIPS}}& \multicolumn{3}{c}{\textbf{HaarPSI}} \\
\cmidrule(lr){4-6}
\cmidrule(lr){7-9}
\cmidrule(lr){10-12}
\cmidrule(lr){13-15}
& & & D-inf & A-inf & $\mathcal{{P}}$-inf & D-inf & A-inf & $\mathcal{{P}}$-inf & D-inf & A-inf & $\mathcal{{P}}$-inf & D-inf & A-inf & $\mathcal{{P}}$-inf \\
\midrule
\hline
\multicolumn{15}{c}{\textbf{16 angles}} \\
\hline
\rowcolor{plotblue!40}$\mathcal{P}_{\text{lat}}^{2}$ & \cmark & \cmark & 27.474 & 27.555 & 27.528 & 0.696 & 0.717 & 0.723 & 0.330 & 0.313 & 0.308 & 0.536 & 0.538 & 0.537 \\
$\mathcal{P}_{\text{lat}}^{2}$ & \xmark & \cmark & 27.185 & 27.402 & 27.315 & 0.687 & 0.714 & 0.720 & 0.333 & 0.313 & 0.319 & 0.526 & 0.531 & 0.525 \\
$\mathcal{P}_{\text{lat}}^{2}$ & \xmark & \xmark & 22.676 & 25.315 & 26.271 & 0.636 & 0.613 & 0.686 & 0.396 & 0.397 & 0.339 & 0.420 & 0.472 & 0.512 \\
\hdashline
$\mathcal{P}_{\text{lat}}^{3}$ & \xmark & \cmark & 27.323 & 27.479 & 27.448 & 0.699 & 0.712 & 0.720 & 0.331 & 0.319 & 0.303 & 0.527 & 0.534 & 0.532 \\
$\mathcal{P}_{\text{lat}}^{3}$ & \xmark & \xmark & 20.666 & 26.593 & 26.585 & 0.459 & 0.699 & 0.710 & 0.446 & 0.304 & 0.311 & 0.258 & 0.528 & 0.524 \\
\hdashline
$\mathcal{P}_{\text{lat}}^{4}$ & \xmark & \cmark & 27.385 & 27.417 & 27.432 & 0.703 & 0.710 & 0.718 & 0.332 & 0.323 & 0.307 & 0.530 & 0.530 & 0.531 \\
\hline
$\mathcal{P}_{\text{ang}}^{2}$ & \xmark & \cmark & 25.803 & 26.389 & 26.321 & 0.664 & 0.688 & 0.693 & 0.335 & 0.341 & 0.345 & 0.446 & 0.455 & 0.451 \\
$\mathcal{P}_{\text{ang}}^{2}$ & \xmark & \xmark & 19.256 & 25.932 & 26.221 & 0.597 & 0.665 & 0.678 & 0.338 & 0.331 & 0.319 & 0.408 & 0.460 & 0.467 \\
\hdashline
\rowcolor{plotblue!10}$\mathcal{P}_{\text{ang}}^{3}$ & \cmark & \cmark & 26.393 & 26.848 & 26.822 & 0.668 & 0.692 & 0.701 & 0.335 & 0.323 & 0.315 & 0.478 & 0.489 & 0.485 \\
$\mathcal{P}_{\text{ang}}^{3}$ & \xmark & \cmark & 26.295 & 26.740 & 26.685 & 0.667 & 0.690 & 0.697 & 0.317 & 0.319 & 0.320 & 0.476 & 0.481 & 0.476 \\
$\mathcal{P}_{\text{ang}}^{3}$ & \xmark & \xmark & 22.862 & 25.191 & 25.500 & 0.618 & 0.645 & 0.659 & 0.336 & 0.342 & 0.337 & 0.413 & 0.441 & 0.447 \\
\hdashline
$\mathcal{P}_{\text{ang}}^{4}$ & \xmark & \cmark & 26.335 & 26.645 & 26.642 & 0.670 & 0.685 & 0.693 & 0.320 & 0.316 & 0.310 & 0.477 & 0.483 & 0.480 \\
\hline
$\mathcal{P}_{\text{mult}}^{2}$ & \xmark & \cmark & 25.975 & 27.370 & 27.311 & 0.671 & 0.714 & 0.727 & 0.317 & 0.321 & 0.319 & 0.452 & 0.494 & 0.501 \\
\hdashline
\rowcolor{plotblue!10}$\mathcal{P}_{\text{mult}}^{3}$ & \xmark & \cmark & 26.824 & 27.349 & 27.382 & 0.704 & 0.705 & 0.724 & 0.295 & 0.309 & 0.299 & 0.487 & 0.499 & 0.511 \\
\hdashline
$\mathcal{P}_{\text{mult}}^{4}$ & \xmark & \cmark & 27.174 & 27.239 & 27.338 & 0.706 & 0.699 & 0.722 & 0.288 & 0.304 & 0.292 & 0.504 & 0.491 & 0.510 \\
\midrule
\hline
\multicolumn{15}{c}{\textbf{32 angles}} \\
\hline
$\mathcal{P}_{\text{lat}}^{2}$ & \xmark & \cmark & 30.432 & 30.741 & 30.512 & 0.768 & 0.788 & 0.790 & 0.223 & 0.230 & 0.236 & 0.666 & 0.669 & 0.662 \\
$\mathcal{P}_{\text{lat}}^{2}$ & \xmark & \xmark & 20.001 & 29.766 & 29.789 & 0.588 & 0.757 & 0.769 & 0.338 & 0.257 & 0.260 & 0.464 & 0.631 & 0.634 \\
\hdashline
$\mathcal{P}_{\text{lat}}^{3}$ & \xmark & \cmark & 30.814 & 30.898 & 30.666 & 0.784 & 0.790 & 0.793 & 0.212 & 0.213 & 0.216 & 0.672 & 0.675 & 0.669 \\
$\mathcal{P}_{\text{lat}}^{3}$ & \xmark & \xmark & 25.446 & 30.313 & 30.006 & 0.635 & 0.781 & 0.775 & 0.285 & 0.246 & 0.249 & 0.503 & 0.650 & 0.642 \\
\hdashline
\rowcolor{plotblue!40}$\mathcal{P}_{\text{lat}}^{4}$ & \cmark & \cmark & 30.926 & 30.983 & 30.811 & 0.788 & 0.792 & 0.795 & 0.215 & 0.217 & 0.220 & 0.677 & 0.678 & 0.673 \\
$\mathcal{P}_{\text{lat}}^{4}$ & \xmark & \cmark & 30.817 & 30.907 & 30.706 & 0.785 & 0.789 & 0.793 & 0.217 & 0.215 & 0.219 & 0.673 & 0.675 & 0.669 \\
\hline
$\mathcal{P}_{\text{ang}}^{2}$ & \xmark & \cmark & 29.493 & 29.843 & 29.557 & 0.758 & 0.769 & 0.767 & 0.241 & 0.265 & 0.283 & 0.607 & 0.606 & 0.599 \\
$\mathcal{P}_{\text{ang}}^{2}$ & \xmark & \xmark & 19.534 & 29.897 & 29.684 & 0.727 & 0.778 & 0.777 & 0.252 & 0.249 & 0.272 & 0.565 & 0.631 & 0.621 \\
\hdashline
$\mathcal{P}_{\text{ang}}^{3}$ & \xmark & \cmark & 29.823 & 30.178 & 29.888 & 0.765 & 0.774 & 0.772 & 0.239 & 0.248 & 0.267 & 0.621 & 0.627 & 0.620 \\
$\mathcal{P}_{\text{ang}}^{3}$ & \xmark & \xmark & 24.858 & 28.556 & 28.666 & 0.725 & 0.748 & 0.756 & 0.228 & 0.223 & 0.245 & 0.578 & 0.611 & 0.608 \\
\hdashline
\rowcolor{plotblue!10}$\mathcal{P}_{\text{ang}}^{4}$ & \cmark & \cmark & 30.019 & 30.266 & 30.021 & 0.770 & 0.776 & 0.775 & 0.229 & 0.239 & 0.257 & 0.629 & 0.633 & 0.627 \\
$\mathcal{P}_{\text{ang}}^{4}$ & \xmark & \cmark & 29.950 & 30.159 & 29.949 & 0.768 & 0.774 & 0.773 & 0.236 & 0.240 & 0.257 & 0.625 & 0.630 & 0.625 \\
\hline
$\mathcal{P}_{\text{mult}}^{2}$ & \xmark & \cmark & 29.292 & 30.797 & 30.434 & 0.764 & 0.794 & 0.793 & 0.280 & 0.255 & 0.261 & 0.585 & 0.645 & 0.642 \\
\hdashline
\rowcolor{plotblue!10}$\mathcal{P}_{\text{mult}}^{3}$ & \xmark & \cmark & 30.422 & 30.755 & 30.492 & 0.784 & 0.790 & 0.791 & 0.240 & 0.241 & 0.248 & 0.636 & 0.646 & 0.646 \\
\hdashline
$\mathcal{P}_{\text{mult}}^{4}$ & \xmark & \cmark & 30.488 & 30.558 & 30.349 & 0.788 & 0.786 & 0.788 & 0.236 & 0.245 & 0.247 & 0.636 & 0.632 & 0.639 \\
\midrule
\bottomrule
\end{tabular}
}\vspace{0.1cm}
\caption{Performance metrics of three splitting methods on uncorrelated noise data (6000 photon counts). Results are reported for different grid sizes, interpolation (yes \cmark / no \xmark), regular vs. irregular splitting, and inference strategies, separately for 16 and 32 projection angles.}
\label{tab:uncorrelated_6000}
\end{table*}

\begin{table*}[ht]
\centering
\small
\setlength{\tabcolsep}{3.5pt}
\resizebox{\textwidth}{!}{
\begin{tabular}{lc|ccc|ccc|ccc|ccc}
\toprule\toprule
& & \multicolumn{12}{c}{\textbf{Performance metrics}} \\
\cmidrule(lr){3-14}
\textbf{Split} & \textbf{Irregular} & \multicolumn{3}{c}{\textbf{PSNR}}& \multicolumn{3}{c}{\textbf{SSIM}}& \multicolumn{3}{c}{\textbf{LPIPS}}& \multicolumn{3}{c}{\textbf{HaarPSI}} \\
\cmidrule(lr){3-5}
\cmidrule(lr){6-8}
\cmidrule(lr){9-11}
\cmidrule(lr){12-14}
& & D-inf & A-inf & $\mathcal{{P}}$-inf & D-inf & A-inf & $\mathcal{{P}}$-inf & D-inf & A-inf & $\mathcal{{P}}$-inf & D-inf & A-inf & $\mathcal{{P}}$-inf \\
\midrule
\hline
\multicolumn{14}{c}{\textbf{16 angles}} \\
\hline
$\mathcal{P}_{\text{lat}}^{2}$ & \xmark & 26.953 & 27.300 & 27.252 & 0.684 & 0.712 & 0.718 & 0.328 & 0.317 & 0.320 & 0.512 & 0.522 & 0.514 \\
\hdashline
\rowcolor{plotblue!40}$\mathcal{P}_{\text{lat}}^{3}$ & \cmark & 27.357 & 27.400 & 27.490 & 0.701 & 0.712 & 0.722 & 0.336 & 0.322 & 0.303 & 0.528 & 0.528 & 0.530 \\
$\mathcal{P}_{\text{lat}}^{3}$ & \xmark & 27.139 & 27.327 & 27.355 & 0.700 & 0.711 & 0.719 & 0.341 & 0.324 & 0.306 & 0.516 & 0.525 & 0.524 \\
\hdashline
$\mathcal{P}_{\text{lat}}^{4}$ & \xmark & 27.133 & 27.280 & 27.315 & 0.700 & 0.707 & 0.716 & 0.327 & 0.319 & 0.304 & 0.513 & 0.520 & 0.521 \\
\hline
$\mathcal{P}_{\text{ang}}^{2}$ & \xmark & 25.560 & 26.308 & 26.257 & 0.654 & 0.684 & 0.691 & 0.323 & 0.338 & 0.333 & 0.436 & 0.455 & 0.451 \\
\hdashline
\rowcolor{plotblue!10}$\mathcal{P}_{\text{ang}}^{3}$ & \cmark & 25.717 & 26.624 & 26.611 & 0.664 & 0.684 & 0.695 & 0.332 & 0.332 & 0.320 & 0.442 & 0.480 & 0.476 \\
$\mathcal{P}_{\text{ang}}^{3}$ & \xmark & 26.118 & 26.429 & 26.406 & 0.662 & 0.688 & 0.694 & 0.325 & 0.318 & 0.323 & 0.467 & 0.464 & 0.461 \\
\hdashline
$\mathcal{P}_{\text{ang}}^{4}$ & \xmark & 26.122 & 26.405 & 26.422 & 0.662 & 0.677 & 0.686 & 0.328 & 0.325 & 0.315 & 0.465 & 0.472 & 0.471 \\
\hline
$\mathcal{P}_{\text{mult}}^{2}$ & \xmark & 25.830 & 27.134 & 27.131 & 0.669 & 0.707 & 0.723 & 0.318 & 0.322 & 0.320 & 0.442 & 0.481 & 0.492 \\
\hdashline
\rowcolor{plotblue!10}$\mathcal{P}_{\text{mult}}^{3}$ & \xmark & 26.644 & 27.060 & 27.246 & 0.699 & 0.698 & 0.719 & 0.302 & 0.314 & 0.299 & 0.477 & 0.480 & 0.504 \\
\hdashline
$\mathcal{P}_{\text{mult}}^{4}$ & \xmark & 26.947 & 26.922 & 27.089 & 0.705 & 0.691 & 0.717 & 0.300 & 0.315 & 0.299 & 0.490 & 0.473 & 0.496 \\
\midrule
\hline
\multicolumn{14}{c}{\textbf{32 angles}} \\
\hline
$\mathcal{P}_{\text{lat}}^{2}$ & \xmark & 30.007 & 30.324 & 30.113 & 0.765 & 0.782 & 0.782 & 0.237 & 0.245 & 0.248 & 0.642 & 0.646 & 0.642 \\
\hdashline
$\mathcal{P}_{\text{lat}}^{3}$ & \xmark & 30.366 & 30.470 & 30.282 & 0.778 & 0.783 & 0.786 & 0.227 & 0.226 & 0.227 & 0.650 & 0.653 & 0.648 \\
\hdashline
\rowcolor{plotblue!40}$\mathcal{P}_{\text{lat}}^{4}$ & \cmark & 30.478 & 30.527 & 30.335 & 0.782 & 0.785 & 0.787 & 0.221 & 0.224 & 0.228 & 0.655 & 0.656 & 0.651 \\
$\mathcal{P}_{\text{lat}}^{4}$ & \xmark & 30.386 & 30.457 & 30.292 & 0.780 & 0.783 & 0.786 & 0.224 & 0.223 & 0.228 & 0.652 & 0.654 & 0.648 \\
\hline
$\mathcal{P}_{\text{ang}}^{2}$ & \xmark & 29.191 & 29.531 & 29.267 & 0.754 & 0.765 & 0.762 & 0.262 & 0.283 & 0.297 & 0.592 & 0.592 & 0.585 \\
\hdashline
$\mathcal{P}_{\text{ang}}^{3}$ & \xmark & 29.471 & 29.748 & 29.536 & 0.759 & 0.767 & 0.766 & 0.253 & 0.260 & 0.274 & 0.601 & 0.610 & 0.601 \\
\hdashline
\rowcolor{plotblue!10}$\mathcal{P}_{\text{ang}}^{4}$ & \cmark & 29.674 & 29.899 & 29.701 & 0.764 & 0.769 & 0.768 & 0.251 & 0.251 & 0.269 & 0.609 & 0.617 & 0.611 \\
$\mathcal{P}_{\text{ang}}^{4}$ & \xmark & 29.587 & 29.822 & 29.603 & 0.761 & 0.768 & 0.766 & 0.243 & 0.252 & 0.268 & 0.607 & 0.611 & 0.607 \\
\hline
$\mathcal{P}_{\text{mult}}^{2}$ & \xmark & 29.370 & 30.325 & 30.012 & 0.755 & 0.786 & 0.785 & 0.245 & 0.260 & 0.266 & 0.600 & 0.626 & 0.625 \\
\hdashline
\rowcolor{plotblue!10}$\mathcal{P}_{\text{mult}}^{3}$ & \xmark & 29.902 & 30.286 & 30.071 & 0.776 & 0.781 & 0.784 & 0.252 & 0.252 & 0.252 & 0.610 & 0.624 & 0.628 \\
\hdashline
$\mathcal{P}_{\text{mult}}^{4}$ & \xmark & 30.079 & 30.151 & 29.989 & 0.780 & 0.777 & 0.782 & 0.241 & 0.251 & 0.256 & 0.621 & 0.616 & 0.623 \\
\midrule
\bottomrule
\end{tabular}
}\vspace{0.1cm}
\caption{Evaluation of splitting methods across grid size, masking, interpolation, and inference types on uncorrelated noise data (3500 photon counts). Results are separated for 16 and 32 angles.}
\label{tab:uncorrelated_3500}
\end{table*}

\begin{table*}[ht]
\centering
\small
\setlength{\tabcolsep}{3.5pt}
\resizebox{\textwidth}{!}{
\begin{tabular}{lc|ccc|ccc|ccc|ccc}
\toprule\toprule
& & \multicolumn{12}{c}{\textbf{Performance metrics}} \\
\cmidrule(lr){3-14}
\textbf{Split} & \textbf{Irregular} & \multicolumn{3}{c}{\textbf{PSNR}}& \multicolumn{3}{c}{\textbf{SSIM}}& \multicolumn{3}{c}{\textbf{LPIPS}}& \multicolumn{3}{c}{\textbf{HaarPSI}} \\
\cmidrule(lr){3-5}
\cmidrule(lr){6-8}
\cmidrule(lr){9-11}
\cmidrule(lr){12-14}
& & D-inf & A-inf & $\mathcal{{P}}$-inf & D-inf & A-inf & $\mathcal{{P}}$-inf & D-inf & A-inf & $\mathcal{{P}}$-inf & D-inf & A-inf & $\mathcal{{P}}$-inf \\
\midrule
\hline
\multicolumn{14}{c}{\textbf{16 angles}} \\
\hline
\rowcolor{plotblue!10}$\mathcal{P}_{\text{lat}}^{2}$ & \cmark & 24.097 & 24.821 & 24.892 & 0.516 & 0.572 & 0.586 & 0.408 & 0.401 & 0.404 & 0.363 & 0.390 & 0.392 \\
$\mathcal{P}_{\text{lat}}^{2}$ & \xmark & 23.900 & 24.484 & 24.566 & 0.510 & 0.567 & 0.577 & 0.409 & 0.416 & 0.425 & 0.349 & 0.373 & 0.377 \\
\hdashline
$\mathcal{P}_{\text{lat}}^{3}$ & \xmark & 23.921 & 24.426 & 24.551 & 0.526 & 0.558 & 0.573 & 0.428 & 0.411 & 0.413 & 0.350 & 0.370 & 0.374 \\
\hdashline
$\mathcal{P}_{\text{lat}}^{4}$ & \xmark & 23.962 & 24.223 & 24.444 & 0.516 & 0.538 & 0.559 & 0.423 & 0.414 & 0.417 & 0.352 & 0.359 & 0.366 \\
\hline
\rowcolor{plotblue!10}$\mathcal{P}_{\text{ang}}^{2}$ & \cmark & 24.762 & 25.417 & 25.527 & 0.629 & 0.661 & 0.672 & 0.379 & 0.394 & 0.363 & 0.395 & 0.408 & 0.413 \\
$\mathcal{P}_{\text{ang}}^{2}$ & \xmark & 24.714 & 25.269 & 25.249 & 0.631 & 0.658 & 0.667 & 0.367 & 0.379 & 0.363 & 0.390 & 0.399 & 0.397 \\
\hdashline
$\mathcal{P}_{\text{ang}}^{3}$ & \xmark & 24.953 & 25.393 & 25.418 & 0.624 & 0.649 & 0.662 & 0.388 & 0.381 & 0.354 & 0.405 & 0.415 & 0.413 \\
\hdashline
$\mathcal{P}_{\text{ang}}^{4}$ & \xmark & 25.029 & 25.308 & 25.342 & 0.626 & 0.645 & 0.659 & 0.383 & 0.373 & 0.344 & 0.411 & 0.417 & 0.414 \\
\hline
\rowcolor{plotblue!40}$\mathcal{P}_{\text{mult}}^{2}$ & \xmark & 24.869 & 25.516 & 25.757 & 0.619 & 0.630 & 0.662 & 0.351 & 0.371 & 0.348 & 0.395 & 0.403 & 0.431 \\
\hdashline
$\mathcal{P}_{\text{mult}}^{3}$ & \xmark & 25.057 & 25.004 & 25.574 & 0.624 & 0.606 & 0.654 & 0.347 & 0.377 & 0.347 & 0.403 & 0.378 & 0.422 \\
\hdashline
$\mathcal{P}_{\text{mult}}^{4}$ & \xmark & 25.083 & 24.732 & 25.599 & 0.609 & 0.587 & 0.649 & 0.352 & 0.388 & 0.336 & 0.406 & 0.369 & 0.428 \\
\midrule
\hline
\multicolumn{14}{c}{\textbf{32 angles}} \\
\hline
\rowcolor{plotblue!10}$\mathcal{P}_{\text{lat}}^{2}$ & \cmark & 23.732 & 25.384 & 25.853 & 0.511 & 0.569 & 0.591 & 0.422 & 0.390 & 0.368 & 0.345 & 0.411 & 0.423 \\
$\mathcal{P}_{\text{lat}}^{2}$ & \xmark & 23.395 & 25.112 & 25.339 & 0.513 & 0.547 & 0.568 & 0.433 & 0.393 & 0.405 & 0.315 & 0.388 & 0.392 \\
\hdashline
$\mathcal{P}_{\text{lat}}^{3}$ & \xmark & 23.702 & 24.590 & 25.195 & 0.518 & 0.535 & 0.559 & 0.417 & 0.399 & 0.397 & 0.332 & 0.363 & 0.386 \\
\hdashline
$\mathcal{P}_{\text{lat}}^{4}$ & \xmark & 23.731 & 23.964 & 24.862 & 0.504 & 0.516 & 0.544 & 0.413 & 0.410 & 0.410 & 0.340 & 0.347 & 0.369 \\
\hline
$\mathcal{P}_{\text{ang}}^{2}$ & \xmark & 27.416 & 27.641 & 27.491 & 0.720 & 0.731 & 0.730 & 0.306 & 0.322 & 0.336 & 0.502 & 0.499 & 0.494 \\
\hdashline
\rowcolor{plotblue!40}$\mathcal{P}_{\text{ang}}^{3}$ & \cmark & 27.907 & 28.101 & 27.951 & 0.728 & 0.736 & 0.736 & 0.274 & 0.288 & 0.305 & 0.528 & 0.528 & 0.525 \\
$\mathcal{P}_{\text{ang}}^{3}$ & \xmark & 27.664 & 27.824 & 27.694 & 0.720 & 0.730 & 0.730 & 0.286 & 0.295 & 0.311 & 0.517 & 0.516 & 0.512 \\
\hdashline
$\mathcal{P}_{\text{ang}}^{4}$ & \xmark & 27.485 & 27.809 & 27.616 & 0.722 & 0.728 & 0.729 & 0.286 & 0.286 & 0.309 & 0.505 & 0.516 & 0.510 \\
\hline
\rowcolor{plotblue!10}$\mathcal{P}_{\text{mult}}^{2}$ & \xmark & 27.238 & 27.187 & 27.534 & 0.711 & 0.675 & 0.693 & 0.297 & 0.348 & 0.298 & 0.496 & 0.482 & 0.524 \\
\hdashline
$\mathcal{P}_{\text{mult}}^{3}$ & \xmark & 27.228 & 26.886 & 27.453 & 0.703 & 0.666 & 0.695 & 0.299 & 0.341 & 0.292 & 0.501 & 0.467 & 0.516 \\
\hdashline
$\mathcal{P}_{\text{mult}}^{4}$ & \xmark & 27.012 & 26.716 & 27.286 & 0.664 & 0.634 & 0.683 & 0.291 & 0.325 & 0.301 & 0.493 & 0.457 & 0.506 \\
\midrule
\bottomrule
\end{tabular}
}\vspace{0.1cm}
\caption{Evaluation of splitting methods across grid size, masking, interpolation, and inference types for correlated noise. Results are separated for 16 and 32 angles.}
\label{tab:correlated}
\end{table*}

\begin{figure*}
    \centering
    \includegraphics[width=1.01\linewidth]{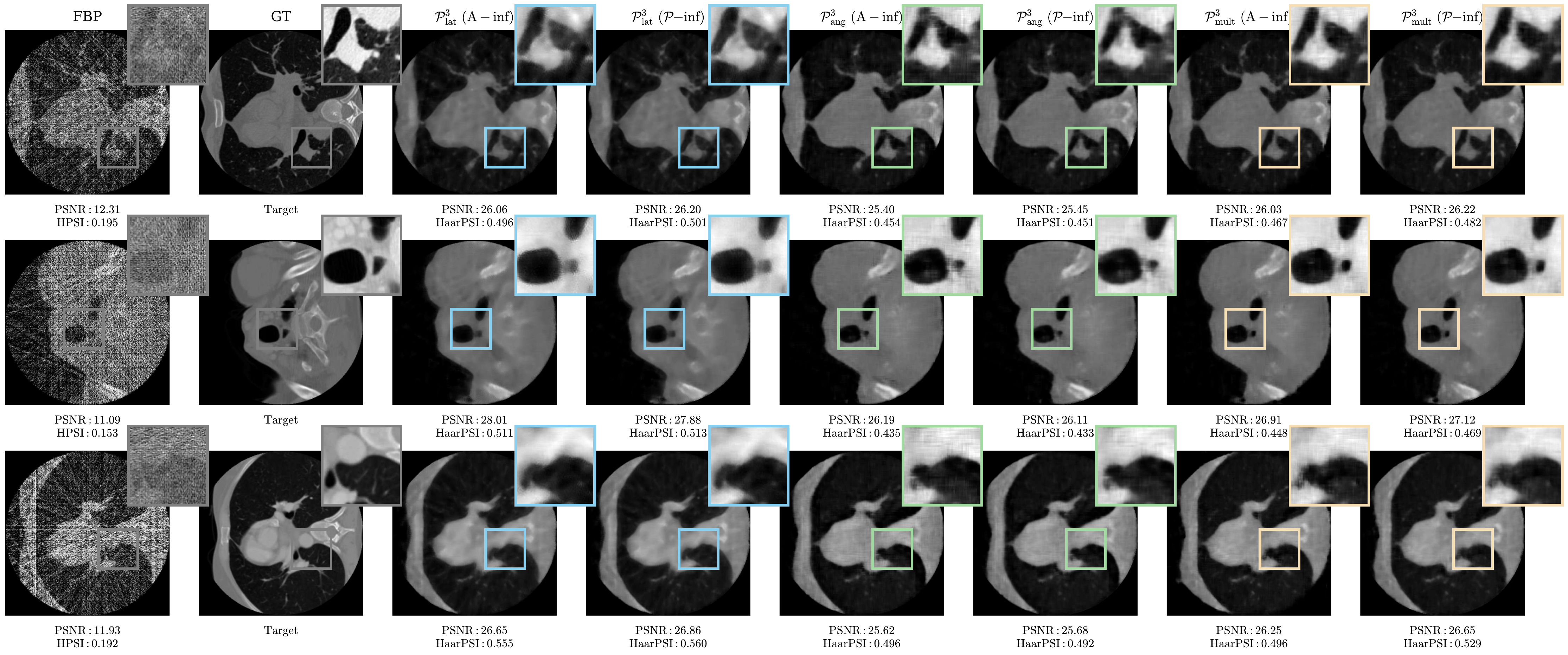}
\caption{Reconstruction results on three examples of the LoDoPaB-CT dataset for 16 projection angles with Poisson noise at a photon count of 3500. For each splitting scheme, the shown reconstructions correspond to the best-performing configuration highlighted in Table~\ref{tab:uncorrelated_3500}. Columns 3--4 show $\mathcal{P}_{\textrm{lat}}^3$ (irregular splitting), columns 5-6 present $\mathcal{P}_{\textrm{ang}}^3$ (irregular splitting), and columns 7-8 show $\mathcal{P}_{\textrm{mult}}^3$. The best-performing inference strategies A-inf and $\mathcal{P}$-inf is displayed across all configurations.}
    \label{reco_16_3500_uncorr}
\end{figure*}

\begin{figure*}
    \centering
    \includegraphics[width=1.01\linewidth]{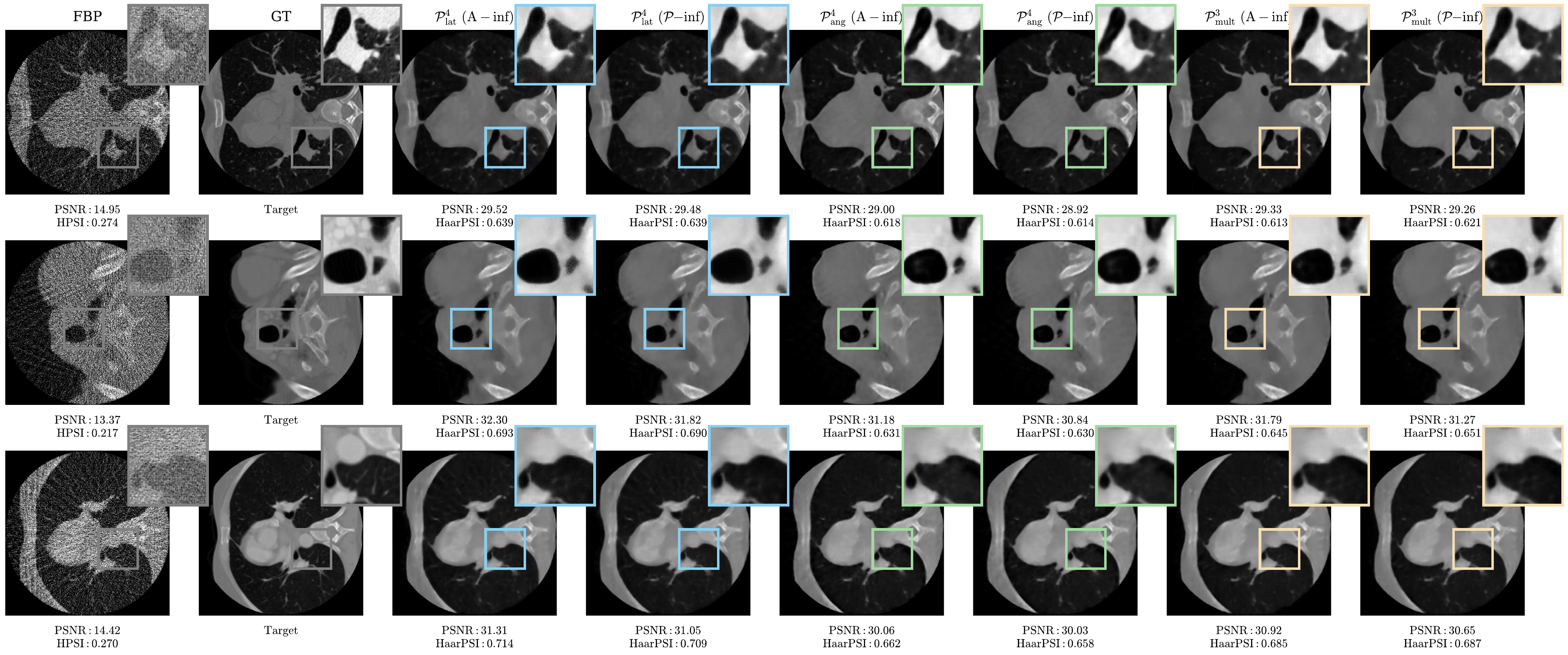}
\caption{Reconstruction results on three examples of the LoDoPaB-CT dataset for 32 projection angles with Poisson noise at a photon count of 3500. For each splitting scheme, the shown reconstructions correspond to the best-performing configuration highlighted in Table~\ref{tab:uncorrelated_3500}. columns 2-3 show $\mathcal{P}_{\textrm{lat}}^4$ ($s=16$, irregular splitting), columns 4-5 present $\mathcal{P}_{\textrm{ang}}^4$ and the last two columns depict $\mathcal{P}_{\textrm{mult}}^3$ (irregular splitting). The best-performing inference strategies A-inf and $\mathcal{P}$-inf is displayed across all configurations.}
    \label{reco_32_3500_uncorr}
\end{figure*}

\begin{figure*}
    \centering
    \includegraphics[width=0.9999\linewidth]{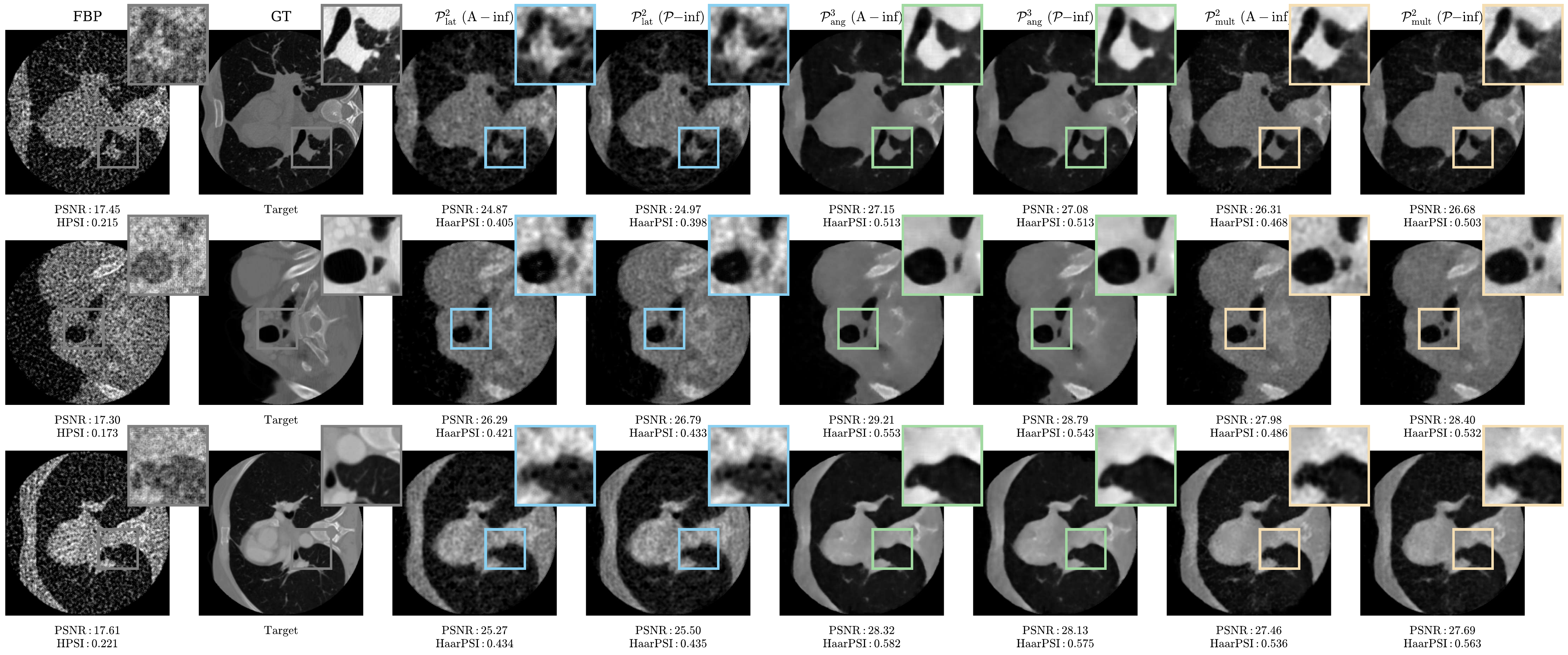}
\caption{Reconstruction results on three examples of the LoDoPaB-CT dataset for 32 projection angles with correlated noise. For each splitting scheme, the shown reconstructions correspond to the best-performing configuration highlighted in Table~\ref{tab:correlated}. columns 2-3 show $\mathcal{P}_{\textrm{lat}}^2$ (irregular splitting), columns 4-5 present $\mathcal{P}_{\textrm{ang}}^3$ (irregular splitting) and the last two columns show $\mathcal{P}_{\textrm{mult}}^2$. The best-performing inference strategies A-inf and $\mathcal{P}$-inf is displayed across all configurations.}
    \label{reco_32_corr}
\end{figure*}

\begin{figure*}
    \centering
    \includegraphics[width=1.01\linewidth]{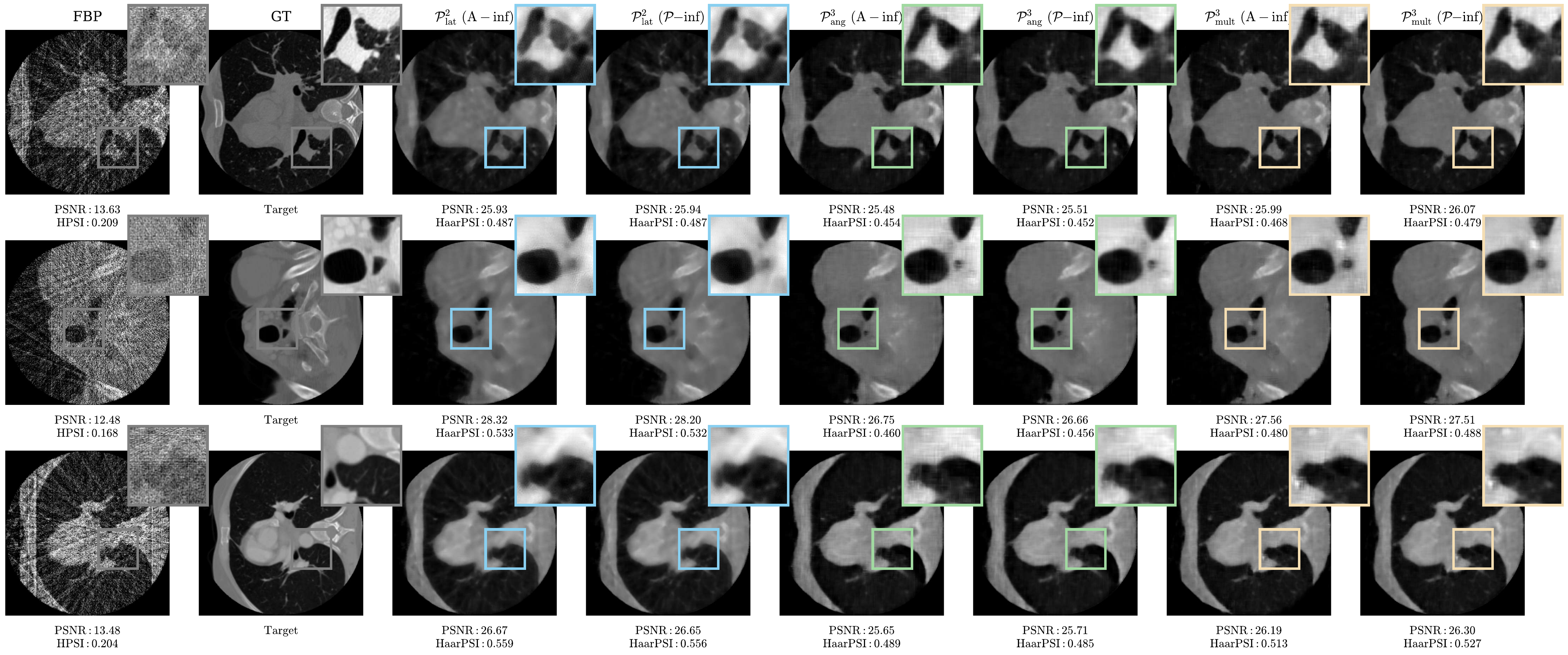}
\caption{Reconstruction results on three examples of the LoDoPaB-CT dataset for 16 projection angles with Poisson noise at a photon count of 6000. For each splitting scheme, the shown reconstructions correspond to the best-performing configuration highlighted in Table~\ref{tab:uncorrelated_6000}. columns 2-3 show $\mathcal{P}_{\textrm{lat}}^2$ (irregular splitting),  columns 5-6 show $\mathcal{P}_{\textrm{ang}}^4$ (irregular splitting), and columns 7-8 present $\mathcal{P}_{\textrm{mult}}^3$. The best-performing inference strategies A-inf and $\mathcal{P}$-inf is displayed across all configurations.}
    \label{reco_16_6000_uncorr}
\end{figure*}

\begin{figure*}[ht]
    \centering
\includegraphics[width=0.995\linewidth]{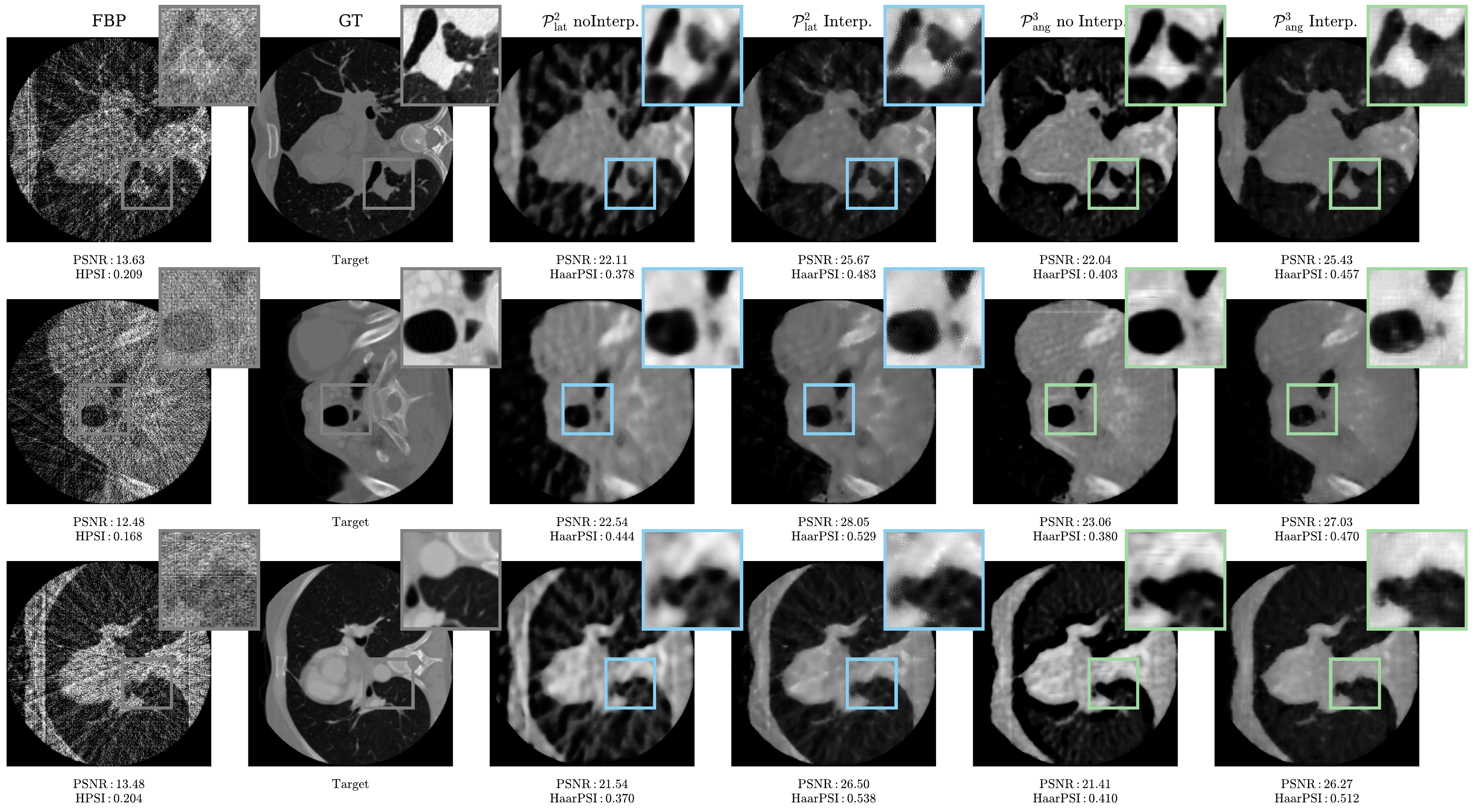}
    \caption{Reconstruction results on the LoDoPaB-CT dataset for 16 projection angles with Poisson noise (intensity 6000). We compare $\mathcal{P}_{\text{ang}}$-splitting with and without linear interpolation of masked values before applying FBP.}
    \label{interpolation}
\end{figure*}

\begin{figure}
    \centering
    \includegraphics[width=0.99\linewidth]{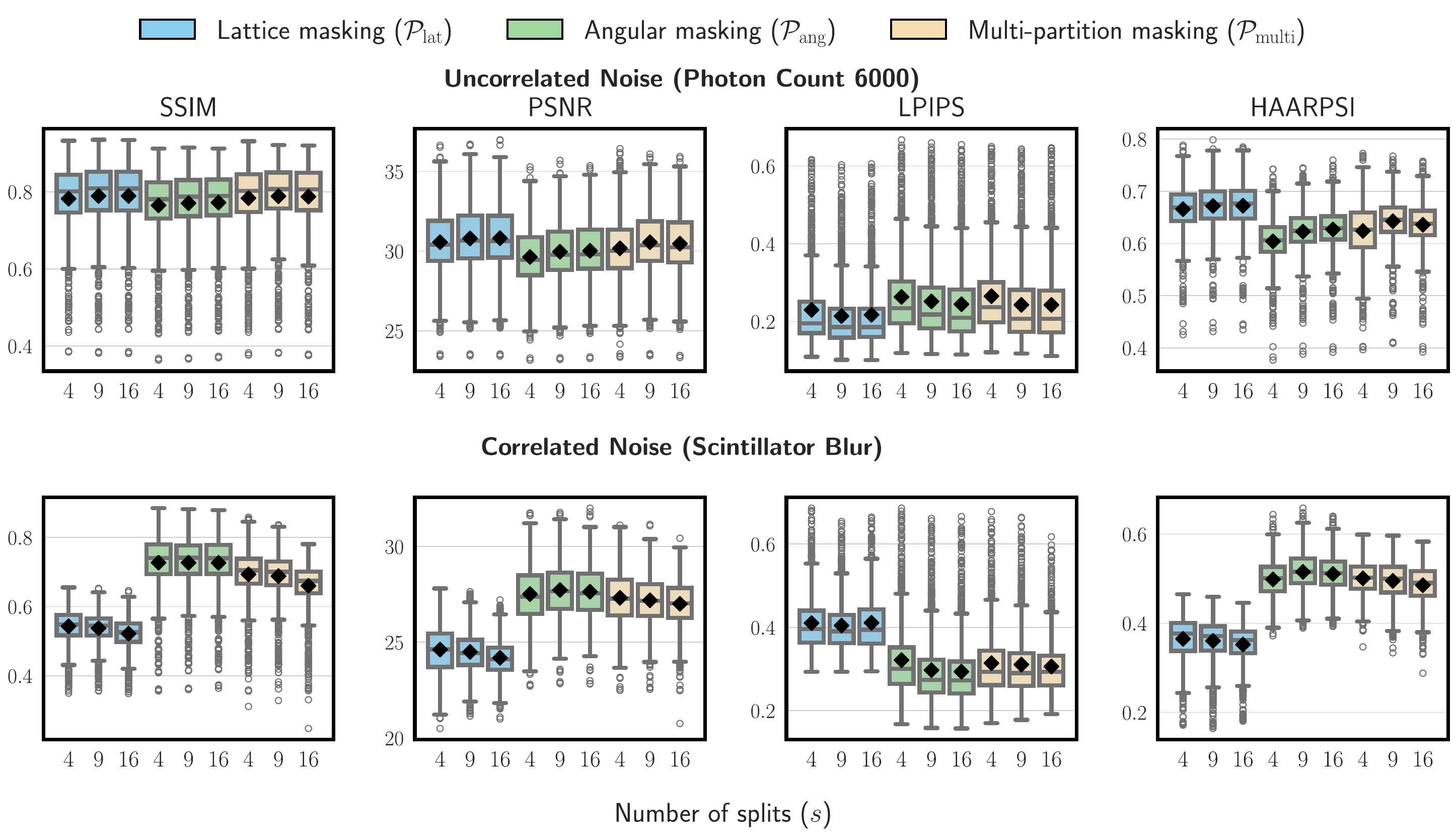} 
    \caption{\textbf{All evaluation metrics for different number of partitions ($s=4,9,16)$ under uncorrelated and correlated noise for $\mathcal{P}-$inf.}}
    \label{gridsizes}
\end{figure}

\begin{figure}
    \centering
    \includegraphics[width=0.8\linewidth]{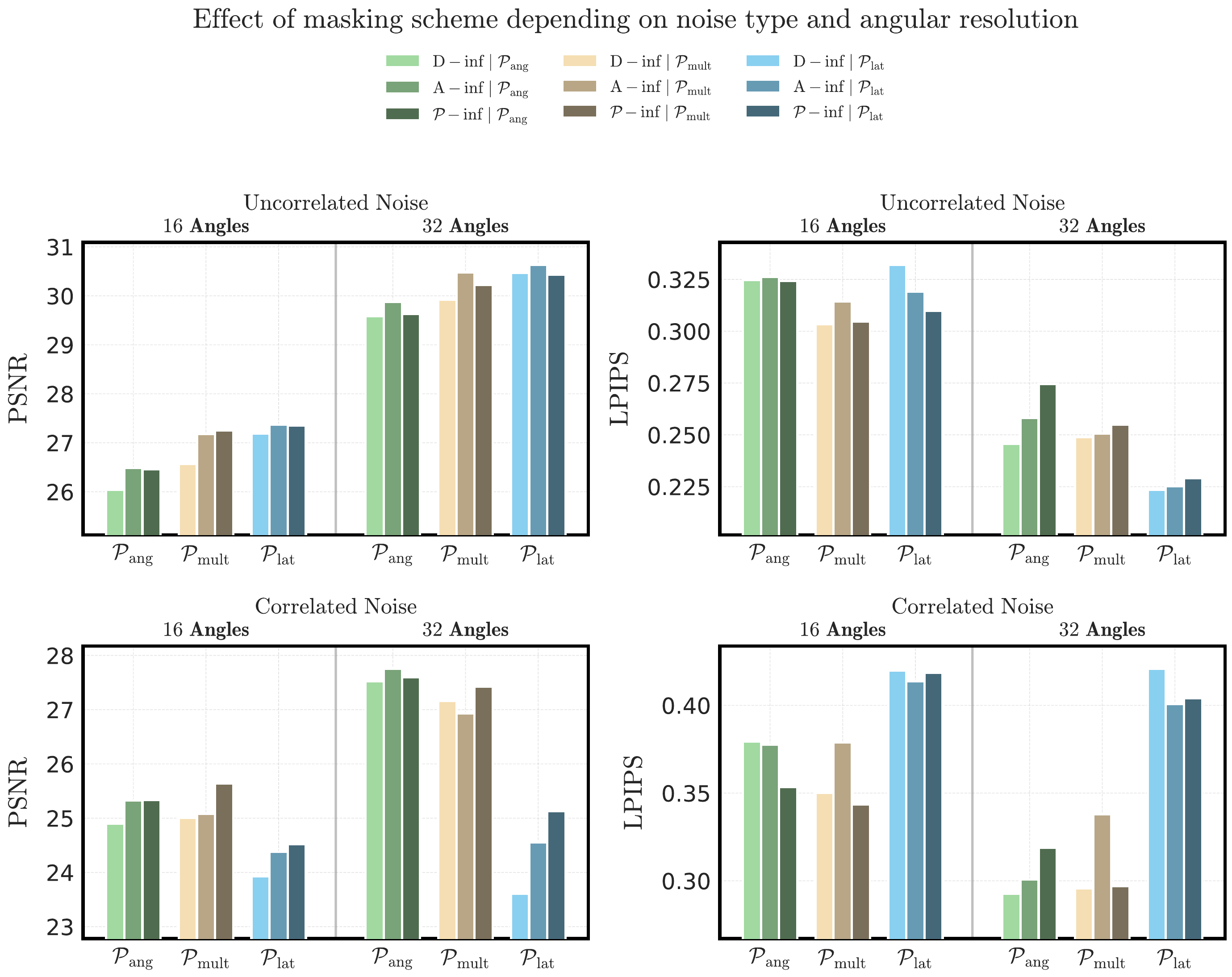} 
\caption{\textbf{PSNR and LPIPS under uncorrelated and correlated noise for the three splitting schemes. For each splitting scheme, the reported values are averaged over the three numbers of splits considered (4, 9, and 16).}}
    \label{lpips_psnr_split}
\end{figure}
\begin{figure}
    \centering
    \includegraphics[width=0.99\linewidth]{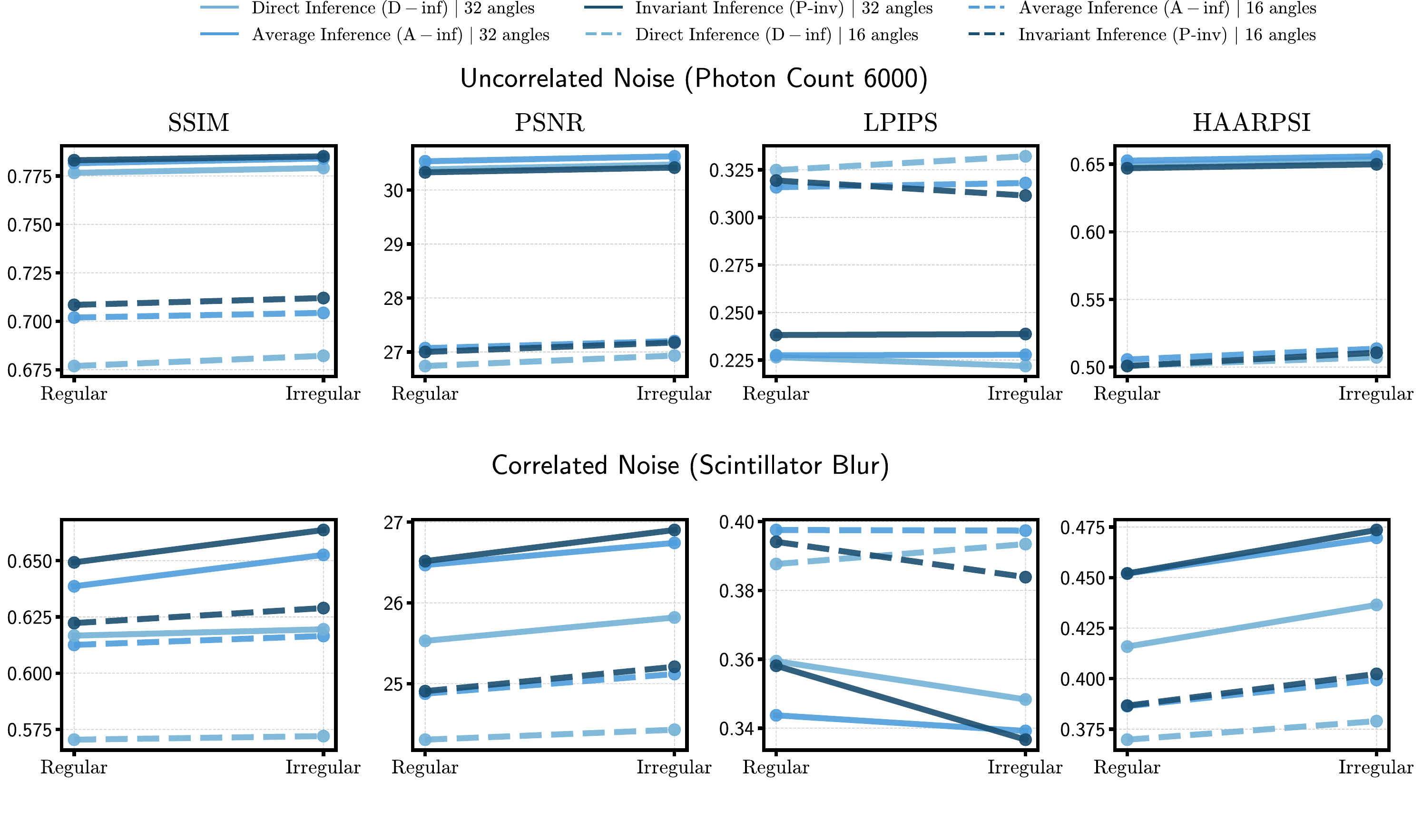} 
    \caption{\textbf{SSIM, PSNR, LPIPS and HaarPSI values showing the slight benefits of irregular splitting during training for the three different inference strategies under uncorrelated and correlated noise.}}
    \label{metrics_structured_vs_unstructured}
\end{figure}

\begin{table*}[ht]
\centering
\small
\setlength{\tabcolsep}{3.5pt}
\resizebox{\textwidth}{!}{
\begin{tabular}{lcc|ccc|ccc|ccc|ccc}
\toprule\toprule
& & & \multicolumn{12}{c}{\textbf{Performance metrics}} \\
\cmidrule(lr){4-15}
\textbf{Split} & \textbf{Irregular} & \textbf{Interp.} & \multicolumn{3}{c}{\textbf{PSNR}}& \multicolumn{3}{c}{\textbf{SSIM}}& \multicolumn{3}{c}{\textbf{LPIPS}}& \multicolumn{3}{c}{\textbf{HaarPSI}} \\
\cmidrule(lr){4-6}
\cmidrule(lr){7-9}
\cmidrule(lr){10-12}
\cmidrule(lr){13-15}
& & & D-inf & A-inf & $\mathcal{{P}}$-inf & D-inf & A-inf & $\mathcal{{P}}$-inf & D-inf & A-inf & $\mathcal{{P}}$-inf & D-inf & A-inf & $\mathcal{{P}}$-inf \\
\midrule
\hline
\multicolumn{15}{c}{\textbf{100 angles}} \\
\hline
$\mathcal{P}_{\text{lat}}^{3}$ & \cmark & \cmark & 23.287 & 23.435 & 23.471 & 0.660 & 0.668 & 0.665 & 0.322 & 0.320 & 0.312 & 0.348 & 0.355 & 0.358 \\
\rowcolor{plotblue!10}$\mathcal{P}_{\text{lat}}^{3}$ & \xmark & \cmark & 23.421 & 23.555 & 23.572 & 0.658 & 0.666 & 0.662 & 0.324 & 0.326 & 0.316 & 0.366 & 0.372 & 0.373 \\
$\mathcal{P}_{\text{lat}}^{3}$ & \xmark & \xmark & 22.586 & 23.523 & 23.528 & 0.575 & 0.699 & 0.683 & 0.356 & 0.330 & 0.345 & 0.312 & 0.344 & 0.345 \\
\hline
\rowcolor{plotblue!40}$\mathcal{P}_{\text{ang}}^{3}$ & \cmark & \cmark & 23.690 & 23.867 & 23.911 & 0.753 & 0.762 & 0.761 & 0.309 & 0.322 & 0.325 & 0.331 & 0.337 & 0.339 \\
$\mathcal{P}_{\text{ang}}^{3}$ & \xmark & \cmark & 23.492 & 23.704 & 23.759 & 0.747 & 0.755 & 0.752 & 0.281 & 0.292 & 0.296 & 0.330 & 0.335 & 0.337 \\
$\mathcal{P}_{\text{ang}}^{3}$ & \xmark & \xmark & 23.462 & 23.448 & 23.507 & 0.744 & 0.739 & 0.740 & 0.270 & 0.291 & 0.296 & 0.335 & 0.333 & 0.334 \\
\hline
$\mathcal{P}_{\text{mult}}^{3}$ & \xmark & \cmark & 23.505 & 23.630 & 23.694 & 0.691 & 0.697 & 0.697 & 0.301 & 0.304 & 0.305 & 0.356 & 0.361 & 0.364 \\
\rowcolor{plotblue!10}$\mathcal{P}_{\text{mult}}^{3}$ & \xmark & \xmark & 23.481 & 23.580 & 23.645 & 0.720 & 0.715 & 0.712 & 0.299 & 0.307 & 0.301 & 0.344 & 0.353 & 0.354 \\
\midrule
\bottomrule
\end{tabular}
}
\caption{Evaluation of splitting methods across masking, interpolation, and inference types on 2DeteCT with 100 angles.}
\label{tab:splitting_results_2detect}
\end{table*}
\end{document}